\documentclass[sn-basic,Numbered]{sn-jnl}

\usepackage{graphicx}
\usepackage{amsmath,amssymb,amsfonts}
\usepackage{amsthm}
\usepackage{bm}
\usepackage{placeins}
\usepackage{booktabs}
\usepackage{xcolor}
\hypersetup{hidelinks,bookmarksdepth=3,
  pdftitle={General Frequency Principle for Linear Approximation Classes},
  pdfauthor={Yujun Wang and Yiqi Gu}}
\theoremstyle{thmstyleone}
\newtheorem{theorem}{Theorem}[section]
\newtheorem{proposition}[theorem]{Proposition}
\newtheorem{lemma}[theorem]{Lemma}
\newtheorem{corollary}[theorem]{Corollary}
\theoremstyle{thmstyletwo}

\numberwithin{equation}{section}
\newcommand{\linearsectionrange}{\ref*{sec:preliminaries}--\ref*{sec:learning}}
\newcommand{\dd}{\mathrm{d}}
\newcommand{\ii}{\mathrm{i}}
\newcommand{\ee}{\mathrm{e}}
\newcommand{\R}{\mathbb{R}}
\newcommand{\C}{\mathbb{C}}
\newcommand{\N}{\mathbb{N}}
\newcommand{\eps}{\varepsilon}

\newcommand{\ipO}[2]{\left\langle #1,#2\right\rangle_{\Omega}}

\newcommand{\ipL}[2]{\left\langle #1,#2\right\rangle_{\Lambda}}
\newcommand{\norm}[1]{\left\|#1\right\|}
\newcommand{\abs}[1]{\left|#1\right|}
\newcommand{\wh}[1]{\widehat{#1}}
\newcommand{\ran}{\operatorname{range}}

\newcommand{\Bk}{\bm{k}}

\begin{document}

\title[General Frequency Principle for Linear Approximation Classes]{General Frequency Principle for Linear Approximation Classes}

\author[1]{\fnm{Yujun} \sur{Wang}}\email{yujunwang@std.uestc.edu.cn}
\author*[1]{\fnm{Yiqi} \sur{Gu}}\email{yiqigu@uestc.edu.cn}

\affil[1]{\orgdiv{School of Mathematical Sciences},
  \orgname{University of Electronic Science and Technology of China},
  \orgaddress{\city{Chengdu}, \postcode{611731}, \country{China}}}

\abstract{In this paper, we study a general frequency principle under gradient descent training for the finite linear approximation classes. Under suitable assumptions, we prove that the learner amplitude relative to the target amplitude vanishes at high frequencies, uniformly throughout training. We also establish sufficient conditions for low-frequency priority, covering single-frequency targets trained separately and two frequency components trained in the same run. This general frequency principle does not imply that lower frequencies always converge faster. Instead, for two frequencies that satisfy certain comparison conditions, one component will take longer to learn than the other; and in common cases, the longer component has higher frequency. Numerical experiments with Legendre and trigonometric bases validate the theoretical results.}


\maketitle

\section{Introduction}
\label{sec:introduction}
Gradient-based training of neural networks often fits low-frequency target components before high-frequency ones. This behavior, known as the frequency principle or spectral bias, has been observed across network architectures and activation functions \cite{xu2019training,rahaman2019spectralbias,xu2020freqprinciple,xu2024overview}. A related question is how frequency-dependent learning can be characterized for linear approximation classes with fixed basis functions. In such classes, a small high-frequency response may result from limited representation or slow convergence toward the best available approximation. Distinguishing these effects is necessary for interpreting the learning order.

One approach to the frequency principle uses Fourier decay. For shallow tanh networks with small input weights, Xu et al.~\cite{xu2020freqprinciple} compared frequency contributions to parameter gradients and, in an idealized two-frequency setting, loss decay rates. Rahaman et al.~\cite{rahaman2019spectralbias} analyzed the Fourier decay of ReLU networks and their parameter derivatives through the piecewise linear structure. Luo et al.~\cite{luo2021theory} estimated high-frequency contributions to loss or output changes under regularity and bounded-trajectory assumptions. Their gradient-flow results cover initial and intermediate training, with final-stage estimates under convergence to a nondegenerate global minimizer. These analyses concern gradients or changes during training, which differ from the time required for an individual Fourier error to reach a prescribed threshold.

A second approach studies the spectrum of the training operator. For fixed linear classes, coefficient gradient descent is Landweber iteration, with convergence governed by the Gram spectrum \cite{landweber1951iteration,engl1996regularization}. Neural tangent kernel (NTK) theory gives a related description of wide networks near their initial linearization \cite{jacot2018ntk,du2019gradient,allen2019convergence,arora2019exact,lee2019wide}. Kernel eigenvalues control residual decay \cite{arora2019fine,cao2021spectralbias}; under uniform spherical data, spherical harmonics relate these rates to frequency \cite{basri2019convergence,cao2021spectralbias}. Nonuniform data density also affects convergence \cite{basri2020frequency}. Fourier descriptions have been derived through mean-field linearization \cite{zhang2019explicitbias} and infinite-width NTK dynamics \cite{luo2022exact}. At finite width, Zhang et al.~\cite{zhang2025shallow} studied Gram spectra, finite-precision approximation limits, and lower bounds on learning times for generalized Fourier coefficients. Bowman and Mont\'ufar~\cite{bowman2022implicit} bounded deviations from NTK dynamics under width and sample-size conditions. For general fixed bases, Fourier modes need not be eigenfunctions of the training operator. Comparing their learning times therefore requires the distribution of each represented mode among the eigendirections, together with its attainable approximation error.

In this paper, we study a general frequency principle for finite linear approximation classes on a bounded box. Let $\Omega=[a,b]^{d}\subset\mathbb R^d$ be a bounded domain and $f\in L^2(\Omega)$ a possibly complex-valued target. We consider using a learner with linear structure to approximate $f$, i.e.,
\begin{equation}\label{h}
h_{\bm c}(\bm x)=\sum_{s=1}^n c_s\phi_s(\bm x), \qquad \bm x\in\Omega,
\end{equation}
where $\phi_1,\ldots,\phi_n\in L^2(\Omega)$ are fixed and linearly independent and only the coefficients $\bm c=(c_1,\ldots,c_n)\in\C^n$ are trained.

Our main contributions are as follows. First, we establish bounds on high-frequency relative movement. Under zero initialization and suitable Fourier decay assumptions, Theorem~\ref{thm:linear-fp} shows that the learner amplitude relative to the target amplitude tends to zero at high frequencies along the prescribed frequency set. This holds at each fixed iteration, and Corollary~\ref{cor:uniform-relative-movement} gives a uniform bound throughout training under a stable step size.

Second, we derive learning-time bounds relative to the best approximation available in the class. Theorem~\ref{thm:comparison} gives sufficient conditions for strict learning order between two single-frequency targets trained separately through comparisons of learning-time bounds, spectral distributions, and spectral means. For lower- and higher-frequency pairs satisfying the corresponding conditions, the lower frequency reaches the prescribed relative error first.

Third, we extend the analysis to general targets. We give sufficient conditions for the order in which two components in the same training run reach and remain below a prescribed relative error. Experiments with Legendre and trigonometric bases illustrate the high-frequency estimates and learning-time results.

These results characterize high-frequency relative movement and conditional learning order within a common linear framework. Here, "general" refers to establishing learning order by comparing bounds: if an upper learning-time bound for one mode is below a lower bound for another, the first mode is learned earlier. Under suitable conditions on the basis, this comparison also implies that the first mode has the smaller frequency magnitude, recovering low-frequency priority for such pairs. Table~\ref{tab:frequency-principle-comparison} compares this analysis with representative results in the literature.

\begin{table}[!htbp]
\centering
\caption{Comparison of representative results on frequency-dependent learning.}
\label{tab:frequency-principle-comparison}
\footnotesize
\setlength{\tabcolsep}{3pt}
\renewcommand{\arraystretch}{1.12}
\begin{tabular}{@{}p{.14\textwidth}p{.29\textwidth}p{.34\textwidth}p{\dimexpr.23\textwidth-18pt\relax}@{}}
\hline
\raggedright Work & \raggedright Setting and conditions & \raggedright Quantities and conclusions & \raggedright Main tools\tabularnewline
\hline
\raggedright Xu et al.~\cite{xu2020freqprinciple} &
\raggedright Shallow tanh networks; small input weights; two-frequency loss for rate comparisons &
\raggedright Low-frequency dominance in parameter gradients and instantaneous loss decay under the stated conditions &
\raggedright Explicit Fourier transforms; small-weight estimates\tabularnewline[5pt]
\raggedright Rahaman et al.~\cite{rahaman2019spectralbias} &
\raggedright ReLU networks with piecewise linear structure &
\raggedright Fourier decay of network outputs and parameter derivatives; implications for high-frequency changes &
\raggedright Fourier analysis; geometry of linear regions\tabularnewline[5pt]
\raggedright Luo et al.~\cite{luo2021theory} &
\raggedright Population risk; regularity and bounded trajectories; nondegenerate minimizer for final-stage estimates &
\raggedright Bounds on high-frequency contributions relative to total loss or output changes at different training stages &
\raggedright Sobolev regularity; Fourier decay; gradient flow\tabularnewline[5pt]
\raggedright Cao et al.~\cite{cao2021spectralbias} &
\raggedright Two-layer ReLU networks; width and sample-size conditions; uniform spherical data for frequency analysis &
\raggedright Residual convergence along NTK eigenspaces; frequency-dependent rates through spherical harmonics &
\raggedright NTK spectra; probability estimates; spherical harmonics\tabularnewline[5pt]
\raggedright Luo et al.~\cite{luo2022exact} &
\raggedright Infinite-width two-layer NTK dynamics; initialization conditions for explicit formulas &
\raggedright Fourier-domain residual dynamics and frequency dependence of the evolution operator &
\raggedright NTK dynamics; Fourier transforms; operator calculations\tabularnewline[5pt]
\raggedright Zhang et al.~\cite{zhang2025shallow} &
\raggedright Finite-width shallow ReLU models; nonlinear bounds use an interval model with fixed input weights, a cutoff, and initialization and weight conditions &
\raggedright Gram spectra and finite-precision limits; time lower bounds for generalized Fourier coefficients of an auxiliary function &
\raggedright Gram spectra; differential operators; gradient flow\tabularnewline[5pt]
\hline
\raggedright This work &
\raggedright Fixed finite linear parameterizations; zero initialization; basis and target Fourier decay for movement bounds; spectral separation for ordering; respective step-size conditions &
\raggedright Capacity and convergence progress in a common framework: high-frequency relative-movement bounds and sufficient conditions for learning-time order after removing the limiting approximation error &
\raggedright Exact discrete dynamics; Fourier decay; projection capacity; local spectral measures and their means\tabularnewline
\hline
\end{tabular}
\end{table}

This paper is organized as follows. Section~\ref{sec:preliminaries} introduces the training operator and its spectral representation. Section~\ref{sec:high-frequency} studies high-frequency relative movement, and Section~\ref{sec:learning} derives learning-time bounds and conditions for learning order. Both sections include numerical experiments. Section~\ref{sec:conclusion} concludes.

\section{Notation and preliminary results}\label{sec:preliminaries}
\makeatletter
\protected@edef\@currentlabel{\protect\linearsectionrange}
\label{sec:linear}
\makeatother
This section introduces the Fourier coordinates and the training operator used throughout the paper. We first identify the attainable least-squares limit, then relate each Fourier mode to the capacity and spectrum of the approximation class.

\subsection{Approximation classes and Fourier coordinates}
Throughout the paper, we use linear structure \eqref{h} to approximate $f$. We call $V=\operatorname{span}\{\phi_1,\ldots,\phi_n\}$ the approximation class and denote by $P_V$ the orthogonal projection operator from $L^2(\Omega)$ onto $V$. Therefore, the projection $P_Vf$ is the part of $f$ that can be truly learned from $V$, which we refer to as the represented part of $f$.

For $u\in L^2(\Omega)$, extend $u$ on $\R^d$ with zero in $\R^d\backslash\Omega$  and define the unitary Fourier transform
\begin{equation}\label{eq01}
\wh u(\Bk)=(2\pi)^{-d/2}\int_{\Omega}u(\bm x)\ee^{-\ii\Bk\cdot\bm x}~\dd\bm x, \qquad \Bk\in\R^d.
\end{equation}
Fix $\vartheta\in\R^d$ and set
\begin{equation}\label{eq:lattice}
\Lambda=\frac{2\pi}{b-a}\left(\mathbb Z^d+\vartheta\right),\qquad \chi_{\Bk}(\bm x)=\abs{\Omega}^{-1/2}\ee^{\ii\Bk\cdot\bm x},\qquad \Bk\in\Lambda .
\end{equation}
For $\Bk\in\Lambda$, $\norm{\Bk}$ denotes its Euclidean norm,
$\norm{\Bk}=(\sum_{j=1}^d\abs{k_j}^2)^{1/2}$. The functions $\{\chi_{\Bk}\}_{\Bk\in\Lambda}$ form an orthonormal basis of $L^2(\Omega)$. We further write
\begin{equation}\label{eq:lattice-coefficient}
u_{\Bk}=\ipO{u}{\chi_{\Bk}}=\left(\frac{2\pi}{b-a}\right)^{d/2}\wh u(\Bk),\qquad \Bk\in\Lambda,\qquad u^\Lambda=(u_{\Bk})_{\Bk\in\Lambda}.
\end{equation}
Thus the lattice coefficients are the values of the Fourier transform \eqref{eq01} on $\Lambda$, up to a fixed factor. By Parseval's identity, the map $u\mapsto u^\Lambda$ is unitary from $L^2(\Omega)$ onto $\ell^2(\Lambda)$, and
\begin{equation}\label{eq:parseval}
\ipL{p^\Lambda}{q^\Lambda}=\sum_{\Bk\in\Lambda}p_{\Bk}\overline{q_{\Bk}}=\ipO{p}{q},\qquad p,q\in L^2(\Omega).
\end{equation}

Now we consider the regression problem. For a target $f\in L^2(\Omega)$ and approximation $h_{\bm c}$, the loss function is given by
\begin{equation}
J(\bm c)=\frac12\norm{h_{\bm c}-f}_{L^2(\Omega)}^2.
\end{equation}
Using gradient descent with learning rate $\eta>0$ leads to
\begin{equation}\label{eq02}
c_s^{(t+1)}=c_s^{(t)}-\eta\ipO{h^{(t)}-f}{\phi_s}, \quad t=0,1,\dots,
\end{equation}
where $c_s^{(t)}$ and $c_s^{(t+1)}$ denote $c_s$ at iterations $t$ and $t+1$, respectively, and $h^{(t)}$ denotes $h_{\bm c}$ at iteration $t$. Note that $h^{(0)}$ is the initialized learner. In matrix form, \eqref{eq02} reads 
\begin{equation}
\bm c^{(t+1)}=\bm c^{(t)}-\eta G\bm c^{(t)}+\eta\bigl(\ipO{f}{\phi_r}\bigr)_{r=1}^n,
\end{equation}
where $G$ is the Gram matrix with entries $G_{rs}=\ipO{\phi_s}{\phi_r}$.

Since $c_{s}\in \mathbb{C}$, we write $c_s=\xi_s+\ii\zeta_s$ with $\xi_s,\zeta_s\in\R$. Differentiation with respect to $\xi_s,\zeta_s$ gives
\begin{equation}
\frac{\partial J}{\partial \xi_s}=\operatorname{Re}\ipO{h_{\bm c}-f}{\phi_s}, \qquad \frac{\partial J}{\partial \zeta_s}=\operatorname{Im}\ipO{h_{\bm c}-f}{\phi_s}.
\end{equation}

\subsection{Training operator and convergence}\label{sec:training-operator}
We express the Fourier error dynamics as an operator iteration and determine its limit. For $s=1,\ldots,n$, let $\phi_s^\Lambda=((\phi_s)_{\Bk})_{\Bk\in\Lambda}$ denote the lattice coefficients \eqref{eq:lattice-coefficient} of the basis functions. Define $\mathcal S:\ell^2(\Lambda)\to\ell^2(\Lambda)$ by
\begin{equation}\label{operator s}
\begin{aligned}
(\mathcal S u^\Lambda)_{\Bk}&=\sum_{\bm l\in\Lambda}\mathcal K_{\Bk\bm l}u_{\bm l},\qquad u\in L^2(\Omega),\\
\mathcal K_{\Bk\bm l}&=\sum_{s=1}^n (\phi_s)_{\Bk}\overline{(\phi_s)_{\bm l}},\qquad \Bk,\bm l\in\Lambda .
\end{aligned}
\end{equation}
The kernel $\mathcal K$ is the Gram matrix of the sequences $\phi_1^\Lambda,\ldots,\phi_n^\Lambda$, indexed by frequency. Equivalently, $\mathcal S$ represents the integral operator on $L^2(\Omega)$ with kernel $\sum_{s=1}^n\phi_s(\bm x)\overline{\phi_s(\bm y)}$ in the orthonormal basis $\{\chi_{\Bk}\}$. A finite-rank singular value expansion gives the following spectral properties \cite{engl1996regularization}.

\begin{proposition}\label{prop01}
The operator $\mathcal S$ is finite-rank, self-adjoint, and positive semidefinite. Moreover, for every $u\in L^2(\Omega)$,
\begin{equation}\label{eq03}
(\mathcal S u^\Lambda)_{\Bk}=\sum_{s=1}^n (\phi_s)_{\Bk}\ipL{u^\Lambda}{\phi_s^\Lambda}, \qquad \ipL{\mathcal S u^\Lambda}{u^\Lambda}=\sum_{s=1}^n\abs{\ipL{u^\Lambda}{\phi_s^\Lambda}}^2.
\end{equation}
In particular, $\ran\mathcal S=\{v^\Lambda:v\in V\}$, and the nonzero eigenvalues of $\mathcal S$ and $G$ coincide.
\end{proposition}

\begin{proof}
By \eqref{operator s}, $\mathcal S u^\Lambda=\sum_{s=1}^n\ipL{u^\Lambda}{\phi_s^\Lambda}\phi_s^\Lambda$, which gives \eqref{eq03}. This representation shows that $\mathcal S$ is self-adjoint and $\ran\mathcal S\subset\{v^\Lambda:v\in V\}$, so $\mathcal S$ is finite-rank. The second identity in \eqref{eq03} gives positive semidefiniteness. 

Note that $\operatorname{rank}G=n$, and $\lambda_1\ge\cdots\ge\lambda_n>0$ be the eigenvalues of $G$ with orthonormal eigenvectors $\bm v_1,\ldots,\bm v_n\in\C^n$. By Gram-Schmit process, we can write an orthonormal basis of $V$, i.e. $\{q_1,\ldots,q_n\}$, where
\begin{equation}\label{eq:qj}
q_j:=\lambda_j^{-1/2}\sum_{s=1}^n(\bm v_j)_s\phi_s,\qquad j=1,\ldots,n.
\end{equation}
Then \eqref{eq03} and \eqref{eq:parseval} give $\mathcal Sq_j^\Lambda=\lambda_j^{-1/2}\sum_{s=1}^n(G\bm v_j)_s\phi_s^\Lambda=\lambda_jq_j^\Lambda$ for all $j$. Also, it is trivial to see that $\mathcal S$ has no nonzero eigenvalues outside of $\{\lambda_j\}_{j=1}^n$. So, $\lambda_1,\dots,\lambda_n$ are exactly the nonzero eigenvalues of $\mathcal S$.
\end{proof}

The coefficient update \eqref{eq02} gives the following iteration for the Fourier error.

\begin{theorem}\label{thm01}
Let $e^{(t)}=(h^{(t)}-f)^\Lambda$ be the sequence of lattice coefficients of the residual $h^{(t)}-f$ at iteration $t$. Then
\begin{equation}\label{eq04}
e^{(t+1)}=(I-\eta\mathcal S)e^{(t)},
\end{equation}
where $I$ is the identity operator on $\ell^2(\Lambda)$.
\end{theorem}

\begin{proof}
By \eqref{eq:parseval} and \eqref{eq02},
\begin{equation}
c_s^{(t+1)}=c_s^{(t)}-\eta\ipO{h^{(t)}-f}{\phi_s}=c_s^{(t)}-\eta\ipL{e^{(t)}}{\phi_s^\Lambda}.
\end{equation}
Therefore, for every $\Bk\in\Lambda$,
\begin{equation}
h^{(t+1)}_{\Bk}=\sum_{s=1}^n c_s^{(t+1)}(\phi_s)_{\Bk}=h^{(t)}_{\Bk}-\eta\sum_{s=1}^n(\phi_s)_{\Bk}\ipL{e^{(t)}}{\phi_s^\Lambda}=h^{(t)}_{\Bk}-\eta(\mathcal S e^{(t)})_{\Bk}.
\end{equation}
Subtracting $f_{\Bk}$ gives $e^{(t+1)}_{\Bk}=e^{(t)}_{\Bk}-\eta(\mathcal S e^{(t)})_{\Bk}$, which is \eqref{eq04}.
\end{proof}

The convergence of gradient descent and the limit follow from the spectral decomposition given as follows.

\begin{theorem}\label{thm:linear-convergence}
Let $P_R$ and $P_N$ denote the orthogonal projections onto the range and the null space of $\mathcal S$, respectively. If $0<\eta<\frac{2}{\lambda_{1}}$, then
\begin{equation}\label{eq:convergence-projection}
\lim_{t\to\infty}P_R e^{(t)}=0, \qquad P_N e^{(t)}=P_N e^{(0)}.
\end{equation}
Moreover, $e^{(t)}$ converges in $\ell^2(\Lambda)$ to
\begin{equation}\label{eq:e-infinity}
e^{(\infty)}=-\left(f-P_Vf\right)^\Lambda .
\end{equation}
In particular, $e^{(t)}\to0$ if and only if $f\in V$.
\end{theorem}

\begin{proof}
By \eqref{eq04}, $e^{(t)}=(I-\eta\mathcal S)^te^{(0)}$. 
Since $q_1^\Lambda,\ldots,q_n^\Lambda$ form an orthonormal basis of $\ran\mathcal S$ and $P_N \mathcal S=0$, we have
\begin{equation}\label{eq:projection-proof}
P_R e^{(t)}=\sum_{j=1}^n(1-\eta\lambda_j)^t\ipL{e^{(0)}}{q_j^\Lambda}q_j^\Lambda, \qquad P_N e^{(t)}=P_Ne^{(0)}.
\end{equation}
The condition  $0<\eta<2/\lambda_{1}$ implies $\abs{1-\eta\lambda_j}<1$ for $1\le j\le n$, which proves \eqref{eq:convergence-projection}. Since $\ran\mathcal S=\{v^\Lambda:v\in V\}$, we have $P_Rw^\Lambda=(P_Vw)^\Lambda$ and $P_Nw^\Lambda=((I-P_V)w)^\Lambda$ for every $w\in L^2(\Omega)$. Finally, $h^{(0)}\in V$ gives $P_Ne^{(0)}=-((I-P_V)f)^\Lambda$, which proves \eqref{eq:e-infinity}.
\end{proof}

The functions $q_j$ are generally not Fourier modes. We therefore write the iteration \eqref{eq04} at a fixed frequency and separate same-frequency damping from frequency mixing.

For a Fourier frequency $\Bk\in\Lambda$, define the functional $\mathcal{M}_{\Bk}:\ell^2(\Lambda)\rightarrow\C$ by
\begin{equation}\label{eq:Skk-R}
\mathcal M_{\Bk}u^\Lambda=(\mathcal S u^\Lambda)_{\Bk}-\mathcal K_{\Bk\Bk}u_{\Bk}=\sum_{\bm l\ne\Bk}\mathcal K_{\Bk\bm l}u_{\bm l},\qquad u\in L^2(\Omega),
\end{equation}
where $\mathcal K_{\Bk\Bk}$ acts as the damping coefficient at frequency $\Bk$, and $\mathcal M_{\Bk}e^{(t)}$ is the mixing contribution from the other residual modes. Equation~\eqref{eq04} gives the exact pointwise decomposition
\begin{equation}\label{eq:point-recursion}
e_{\Bk}^{(t+1)} =(1-\eta \mathcal K_{\Bk\Bk})e_{\Bk}^{(t)}-\eta \mathcal M_{\Bk}e^{(t)}.
\end{equation}
Here $(1-\eta \mathcal K_{\Bk\Bk})e_{\Bk}^{(t)}$ is the same-frequency contribution, and $-\eta\mathcal M_{\Bk}e^{(t)}$ accounts for frequency mixing. Since the lattice modes form an orthonormal basis, \eqref{eq:point-recursion} describes the exact training dynamics without discretization. The Fourier dynamics are diagonal only if $\mathcal K_{\Bk\bm l}=0$ for $\Bk\ne\bm l$, equivalently, if every lattice mode is an eigenvector of $\mathcal S$.

\subsection{Capacity and local spectral measures}\label{sec:capacity}

To distinguish representational capacity from learning speed, we use the orthonormal basis $\{q_j\}$ of $V$ from Proposition~\ref{prop01} and define the capacity at frequency $\Bk$ by
\begin{equation}\label{eq:capacity}
\ell_{\Bk}=\norm{P_V\chi_{\Bk}}_{L^2(\Omega)}^2=\sum_{j=1}^{n}\abs{(q_j)_{\Bk}}^2 \qquad \Bk\in\Lambda,
\end{equation}
where $(q_j)_{\Bk}=\ipO{q_j}{\chi_{\Bk}}$. The capacity satisfies $0\le\ell_{\Bk}\le1$. Since $q_1^\Lambda,\ldots,q_n^\Lambda$ form an orthonormal basis of $\ran\mathcal S$ and satisfy $\mathcal S q_j^\Lambda=\lambda_jq_j^\Lambda$, the self-adjointness of $\mathcal S$ yields
\begin{equation}\label{Sw}
\mathcal S w
=\sum_{j=1}^{n}\lambda_j
\ipL{w}{q_j^\Lambda}q_j^\Lambda,
\qquad w\in\ell^2(\Lambda).
\end{equation}
Applying \eqref{operator s}, \eqref{eq03} and \eqref{Sw} to the unit vector $e_{\Bk}\in\ell^2(\Lambda)$ gives
\begin{equation}\label{eq:Skk-spectral}
\mathcal K_{\Bk\Bk}=\sum_{s=1}^n\abs{(\phi_s)_{\Bk}}^2=\ipL{\mathcal Se_{\Bk}}{e_{\Bk}}=\sum_{j=1}^{n}\lambda_j\abs{(q_j)_{\Bk}}^2.
\end{equation}
Hence
\begin{equation}
\lambda_{n}\ell_{\Bk}\le \mathcal K_{\Bk\Bk} \le\lambda_{1}\ell_{\Bk}.
\end{equation}
For $\ell_{\Bk}>0$, we define the local spectral measure of $\Bk$ through the weights
\begin{equation}\label{eq:local-measure}
\varpi_j(\Bk)=\frac{\abs{(q_j)_{\Bk}}^2}{\ell_{\Bk}},\qquad j=1,\ldots,n.
\end{equation}
By \eqref{eq:capacity}, these weights are nonnegative and sum to one. We call their mean
\begin{equation}\label{eq:speed}
\bar\lambda_{\Bk}=\sum_{j=1}^{n}\varpi_j(\Bk)\lambda_j=\frac{\mathcal K_{\Bk\Bk}}{\ell_{\Bk}}
\end{equation}
the speed at frequency $\Bk$. The second equality follows from \eqref{eq:Skk-spectral}. Thus the same-frequency damping coefficient in \eqref{eq:point-recursion} factors as
\begin{equation}\label{eq:factorization}
\mathcal K_{\Bk\Bk}=\ell_{\Bk}\,\bar\lambda_{\Bk}.
\end{equation}
The speed $\bar\lambda_{\Bk}$ measures the average learning rate of the represented part. Note that $\frac{\partial h_{\bm c}}{\partial c_s}=\phi_s$, the coefficient $\mathcal K_{\Bk\Bk}$ is the sum of the squared moduli of the derivatives of $(h_{\bm c})_{\Bk}$ with respect to $c_s$. Its high-frequency decay underlies the explanation in \cite{rahaman2019spectralbias}. The factorization \eqref{eq:factorization} separates this decay into effects on capacity and speed.

\subsection{Boundary traces and capacity bounds}\label{sec:trace-capacity}
We begin with an integration-by-parts estimate relating the order of the first nonzero boundary trace to the Fourier decay rate. For $j\in\{1,\ldots,d\}$ and $\bm x\in\mathbb{R}^d$, let $\bm x'_j$ denote the vector obtained from $\bm x$ by removing its $j$-th entry, and set $\Omega'_j=[a,b]^{d-1}$. For $t\in[a,b]$ and $\bm x'_j\in\Omega'_j$, we use the shorthand
\begin{equation}
u(t,\bm x'_j):=u(x_1,\ldots,x_{j-1},t,x_{j+1},\ldots,x_d),
\end{equation}
which replaces the $j$-th coordinate by $t$ while keeping the other coordinates unchanged. We apply one-dimensional integration by parts in $x_j$ with the remaining variables fixed. For a nonnegative integer $m$, let $W^{m,1}(\Omega)$ denote the Sobolev space of functions in $L^1(\Omega)$ whose weak derivatives up to order $m$ also belong to $L^1(\Omega)$.

\begin{lemma}\label{lem:trace-decay}
Let $\nu\in\N_0=\{0,1,2,\ldots\}$ and $u\in W^{\nu+1,1}(\Omega)$. Assume that, for every $j\in\{1,\ldots,d\}$ and every integer $0\le p<\nu$, 
\begin{equation}\label{eq:trace-vanish}
\partial_{x_j}^p u(a,\bm x'_j)=\partial_{x_j}^p u(b,\bm x'_j)=0 \qquad\text{for almost every }\bm x'_j\in\Omega'_j.
\end{equation}
Then there is $B_u>0$, depending only on $u$, such that
\begin{equation}\label{eq:u-decay}
\abs{\wh u(\Bk)}\le B_u\norm{\Bk}^{-(\nu+1)}, \qquad \Bk\in\R^d\setminus\{0\}.
\end{equation}
\end{lemma}

\begin{proof}
Let $\Bk\ne0$ and choose $j$ such that $\abs{k_j}=\norm{\Bk}_{\infty}$. Since $u\in W^{\nu+1,1}(\Omega)$, for almost every fixed $\bm x'_j\in\Omega'_j$ the functions $x_j\mapsto\partial_{x_j}^pu(\bm x)$, $0\le p\le\nu$, are absolutely continuous on $[a,b]$ after a change on a null set. Their derivatives are $x_j\mapsto\partial_{x_j}^{p+1}u(\bm x)$, and their values at $a$ and $b$ equal the traces of $\partial_{x_j}^pu$. For such $\bm x'_j$, we integrate by parts $\nu+1$ times in $x_j$:
\begin{multline}
\int_a^b u(\bm x)\ee^{-\ii k_jx_j}~\dd x_j =\sum_{p=0}^{\nu}\frac{(-1)^p}{(-\ii k_j)^{p+1}} \Bigl[\partial_{x_j}^p u(\bm x)\ee^{-\ii k_jx_j}\Bigr]_{x_j=a}^{x_j=b}\\
+\frac{(-1)^{\nu+1}}{(-\ii k_j)^{\nu+1}} \int_a^b\partial_{x_j}^{\nu+1}u(\bm x)\ee^{-\ii k_jx_j}~\dd x_j.
\end{multline}
The boundary terms with $p<\nu$ vanish by \eqref{eq:trace-vanish}. Taking absolute values of the remaining terms and integrating over $\Omega'_j$ gives
\begin{multline}
\abs{\wh u(\Bk)} \le \frac{(2\pi)^{-d/2}}{\abs{k_j}^{\nu+1}} \left(\int_{\Omega'_j} \abs{\partial_{x_j}^{\nu}u(b,\bm x'_j)} +\abs{\partial_{x_j}^{\nu}u(a,\bm x'_j)}~\dd\bm x'_j +\norm{\partial_{x_j}^{\nu+1}u}_{L^1(\Omega)} \right)\\
\le C_u\abs{k_j}^{-(\nu+1)},
\end{multline}
where $C_u$ is $(2\pi)^{-d/2}$ times the maximum over $j=1,\ldots,d$ of the expression in parentheses. Since $\abs{k_j}=\norm{\Bk}_{\infty}\ge d^{-1/2}\norm{\Bk}$, we have
\begin{equation}
\abs{\wh u(\Bk)} \le C_u d^{(\nu+1)/2}\norm{\Bk}^{-(\nu+1)}.
\end{equation}
Setting $B_u=C_u d^{(\nu+1)/2}$ proves \eqref{eq:u-decay}.
\end{proof}

The next proposition uses basis-function decay to bound the capacity \eqref{eq:capacity} and the change of high-frequency modes uniformly in the training time.

\begin{proposition}\label{prop:capacity}
Fix $\nu\in\N_0$ and a set $\mathcal A\subset\Lambda\setminus\{0\}$. Assume that
\begin{equation}\label{eq:basis-decay}
\abs{(\phi_s)_{\Bk}}\le B_\phi\norm{\Bk}^{-(\nu+1)}, \qquad s=1,\ldots,n,\quad \Bk\in\mathcal A,
\end{equation}
where $B_\phi>0$ is independent of $\Bk$. Then
\begin{equation}\label{eq:capacity-bound}
\mathcal K_{\Bk\Bk}\le nB_\phi^2\norm{\Bk}^{-2(\nu+1)},\qquad \ell_{\Bk}\le\frac{nB_\phi^2}{\lambda_{n}}\norm{\Bk}^{-2(\nu+1)},\qquad \Bk\in\mathcal A.
\end{equation}
Moreover, if $h^{(0)}=0$ and $0<\eta<2/\lambda_{1}$, then, for every $t\in\N_0$,
\begin{equation}\label{eq:capacity-decay}
\abs{e_{\Bk}^{(t)}-e_{\Bk}^{(0)}}\le2\norm{f}_{L^2(\Omega)}B_\phi\sqrt{\frac{n}{\lambda_{n}}}\,\norm{\Bk}^{-(\nu+1)},\qquad \Bk\in\mathcal A.
\end{equation}
\end{proposition}

\begin{proof}
The first inequality in \eqref{eq:capacity-bound} follows from \eqref{eq:Skk-spectral} and \eqref{eq:basis-decay}. The second follows from $\ell_{\Bk}\le \mathcal K_{\Bk\Bk}/\lambda_{n}$.

By the proof of Theorem~\ref{thm:linear-convergence}, $(h^{(t)}-P_Vf)^\Lambda=P_Re^{(t)}$. According to \eqref{eq:projection-proof} and the orthonormality of $q_j^{\Lambda}$,
\begin{equation}
\begin{aligned}
\norm{P_Re^{(t)}}_{\ell^2(\Lambda)}
&=\norm{\sum_{j=1}^n(1-\eta\lambda_j)^t\ipL{e^{(0)}}{q_j^\Lambda}q_j^\Lambda}_{\ell^2(\Lambda)}\\
&=\left(\sum_{j=1}^n\abs{1-\eta\lambda_j}^{2t}\abs{\ipL{e^{(0)}}{q_j^\Lambda}}^2\right)^{1/2}\\
&\le\left(\sum_{j=1}^n\abs{\ipL{e^{(0)}}{q_j^\Lambda}}^2\right)^{1/2}\\
&=\norm{P_Re^{(0)}}_{\ell^2(\Lambda)}.
\end{aligned}
\end{equation}
Hence, by Parseval's identity \eqref{eq:parseval} and $h^{(0)}=0$,
\begin{equation}
\begin{aligned}
\norm{h^{(t)}-P_Vf}_{L^2(\Omega)}
&=\left(\sum_{\Bk\in\Lambda}\abs{\ipO{h^{(t)}-P_Vf}{\chi_{\Bk}}}^2\right)^{1/2}\\
&=\norm{(h^{(t)}-P_Vf)^\Lambda}_{\ell^2(\Lambda)}
=\norm{P_Re^{(t)}}_{\ell^2(\Lambda)}\\
&\le\norm{P_Re^{(0)}}_{\ell^2(\Lambda)}
=\norm{(h^{(0)}-P_Vf)^\Lambda}_{\ell^2(\Lambda)}\\
&=\norm{h^{(0)}-P_Vf}_{L^2(\Omega)}
=\norm{P_Vf}_{L^2(\Omega)}.
\end{aligned}
\end{equation}
Consequently,
\begin{equation}
\norm{h^{(t)}}_{L^2(\Omega)}\le\norm{P_Vf}_{L^2(\Omega)}+\norm{h^{(t)}-P_Vf}_{L^2(\Omega)}\le2\norm{P_Vf}_{L^2(\Omega)}\le2\norm{f}_{L^2(\Omega)},\qquad t\in\N_0.
\end{equation}
Since $h^{(0)}=0$, we have
\begin{equation}
\begin{aligned}
\abs{e_{\Bk}^{(t)}-e_{\Bk}^{(0)}}
&=\abs{\ipO{h^{(t)}}{\chi_{\Bk}}}
=\abs{\ipO{h^{(t)}}{P_V\chi_{\Bk}}}\\
&\le
\norm{h^{(t)}}_{L^2(\Omega)}\norm{P_V\chi_{\Bk}}_{L^2(\Omega)}\\
&=\norm{h^{(t)}}_{L^2(\Omega)}\sqrt{\ell_{\Bk}}\\
&\le2\norm{f}_{L^2(\Omega)}B_\phi\sqrt{\frac{n}{\lambda_{n}}}\,\norm{\Bk}^{-(\nu+1)},
\end{aligned}
\end{equation}
This proves \eqref{eq:capacity-decay}.
\end{proof}

The spectral condition \eqref{eq:basis-decay} admit many familiar examples. To illustrate the basis and target assumptions, consider $\Omega=[0,1]$ and $\Lambda=2\pi\mathbb Z$. Any fixed finite family of polynomials, including Legendre and Chebyshev polynomials, or fixed Gaussian functions restricted to this interval satisfies \eqref{eq:hf-basis-decay} with $\nu=0$, with a constant depending on the chosen family.

Proposition~\ref{prop:capacity} gives a bound uniform in $t$, including the limit $t\to\infty$, without further regularity assumptions on the target $f$. The upper bound in \eqref{eq:capacity-decay} is independent of $t$ and depends on $\Bk$ through $\norm{\Bk}^{-(\nu+1)}$, but it does not determine which frequency will be learned first. Moreover, the potential large factor $(n/\lambda_{n})^{1/2}$ can make the bound large at moderate frequencies. 

\subsection{Spectral representation of the attainable error}\label{sec:learning-time}

The next proposition gives a closed-form error for each frequency $\Bk$. Recall that $e^{(\infty)}$ given by \eqref{eq:e-infinity} is best error we can obtain when using $V$ as the approximation space, and $e^{(t)}-e^{(\infty)}$ measures the gap between the current approximator and the best one. 
 
\begin{proposition}\label{prop:per-frequency}
Let $0<\eta<2/\lambda_{1}$. Then, for every $t\in\N_0$ and every $\Bk\in\Lambda$,
\begin{equation}\label{eq:per-frequency}
e_{\Bk}^{(t)}-e_{\Bk}^{(\infty)}=\sum_{j=1}^{n}(1-\eta\lambda_j)^t\ipO{h^{(0)}-f}{q_j}(q_j)_{\Bk}.
\end{equation}
\end{proposition}

\begin{proof}
By the proof of Theorem~\ref{thm:linear-convergence}, $e^{(\infty)}=P_Ne^{(0)}=P_Ne^{(t)}$. Hence $e^{(t)}-e^{(\infty)}=P_Re^{(t)}$, and \eqref{eq:projection-proof} gives
\begin{equation}
e^{(t)}-e^{(\infty)}=\sum_{j=1}^{n}(1-\eta\lambda_j)^t\ipL{e^{(0)}}{q_j^\Lambda}q_j^\Lambda.
\end{equation}
By \eqref{eq:parseval}, $\ipL{e^{(0)}}{q_j^\Lambda}=\ipO{h^{(0)}-f}{q_j}$. Taking the entry at $\Bk$ gives \eqref{eq:per-frequency}.
\end{proof}

Formula~\eqref{eq:per-frequency} shows that the distance between $e^{(t)}_{\Bk}$ and $e^{(\infty)}_{\Bk}$ consist of two parts: $(1-\eta\lambda_j)^t$ involves the time and $(q_j)_{\Bk}$ is the $\Bk$-th Fourier coefficient of $q_j$. Thus $\Bk$ is learned slowly when the part of $\chi_{\Bk}$ in $V$ lies mainly along eigenfunctions $q_j$ with small eigenvalues $\lambda_j$.

\section{High-frequency movement and representational limits}\label{sec:high-frequency}
\subsection{Bounds on high-frequency relative changes}
We bound high-frequency changes in the Fourier error relative to its initial amplitude. Under suitable Fourier decay assumptions, the bounds hold at each fixed iteration and, for stable gradient descent, uniformly throughout training. The uniform bound distinguishes limited approximation capacity from slow convergence.

Let $\mathcal A\subset\{\Bk\in\Lambda:\norm{\Bk}\ge1\}$ be unbounded. For zero initialization, $e_{\Bk}^{(0)}=-f_{\Bk}$. At some frequency $\Bk$ where $f_{\Bk}\ne0$, define the relative change during the training by
\begin{equation}\label{eq:relative-movement}
\Gamma_t(\Bk)
 :=\frac{\abs{h_{\Bk}^{(t)}}}{\abs{f_{\Bk}}}=\frac{\abs{h_{\Bk}^{(t)}-h_{\Bk}^{(0)}}}{\abs{f_{\Bk}}}=\frac{\abs{e_{\Bk}^{(t)}-e_{\Bk}^{(0)}}}
 {\abs{e_{\Bk}^{(0)}}}.
\end{equation}
This quantity measures how far the current learner is from initialization at frequency $\Bk$. Intuitively, small $\Gamma_t(\Bk)$ means the learner is trained slowly at this frequency. From the third equality in \eqref{eq:relative-movement}, small $\Gamma_t(\Bk)$ also means that the Fourier error remains close to its initial value on this relative scale, but it does not assert that the mode is close to its attainable limit. Moreover, the special case  $\Gamma_t(\Bk)=1$ means equal amplitudes of $h_{\Bk}^{(t)}$ and $f_{\Bk}$, but does not ensure the accurate learning, i.e., $h_{\Bk}^{(t)}=f_{\Bk}$.

\begin{theorem}\label{thm:linear-fp}
Let $\nu\in\N_0$ and suppose that, for some $B_\phi>0$,
\begin{equation}\label{eq:hf-basis-decay}
\abs{(\phi_s)_{\Bk}}\le B_\phi\norm{\Bk}^{-(\nu+1)},
\qquad s=1,\ldots,n,\quad \Bk\in\mathcal A.
\end{equation}
Assume also that there are $0<\alpha<\nu+1$ and $0<c_f\le C_f$ such that
\begin{equation}\label{eq:hf-target-two-sided}
c_f\norm{\Bk}^{-\alpha}\le\abs{f_{\Bk}}
\le C_f\norm{\Bk}^{-\alpha},\qquad \Bk\in\mathcal A.
\end{equation}
Let $h^{(0)}=0$ and $\eta>0$. For every fixed integer $t\ge1$, there are constants $C_{1,t},C_{2,t}>0$, independent of $\Bk$, such that
\begin{equation}\label{eq:linear-fp}
\Gamma_t(\Bk)
\le C_{1,t}\norm{\Bk}^{-2(\nu+1)}
   +C_{2,t}\norm{\Bk}^{\alpha-(\nu+1)},
\qquad \Bk\in\mathcal A.
\end{equation}
Consequently, $\Gamma_t(\Bk)\to0$ as $\norm{\Bk}\to\infty$ for every fixed $t$.
\end{theorem}

\begin{proof}
Write $p=\nu+1$. By \eqref{eq:hf-basis-decay},
\begin{equation}\label{eq:hf-diagonal}
0\le\mathcal K_{\Bk\Bk}
=\sum_{s=1}^n\abs{(\phi_s)_{\Bk}}^2
\le nB_\phi^2\norm{\Bk}^{-2p}.
\end{equation}
The off-diagonal functional $\mathcal M_{\Bk}$ in \eqref{eq:point-recursion} satisfies
\begin{equation}\label{eq:hf-mixing-row}
\norm{\mathcal M_{\Bk}}^2
=\sum_{\bm l\ne\Bk}\abs{\mathcal K_{\Bk\bm l}}^2
\le\sum_{\bm l\in\Lambda}\abs{\mathcal K_{\Bk\bm l}}^2
\le\lambda_1\mathcal K_{\Bk\Bk}.
\end{equation}
Here the norm is from $\ell^2(\Lambda)$ to $\C$. The last inequality follows from $\mathcal S^2\le\lambda_1\mathcal S$, applied to the unit coordinate vector at $\Bk$. Hence, with
\begin{equation}
E_t=\max_{0\le q<t}\norm{e^{(q)}}_{\ell^2(\Lambda)},
\end{equation}
we have
\begin{equation}\label{eq:hf-mixing}
\abs{\mathcal M_{\Bk}e^{(q)}}
\le B_\phi\sqrt{n\lambda_1}\,E_t\norm{\Bk}^{-p},
\qquad 0\le q<t.
\end{equation}
The number $E_t$ is finite because $\mathcal S$ is bounded and $t$ is fixed; no stability restriction on $\eta$ is needed here.

Set $a_{\Bk}=1-\eta\mathcal K_{\Bk\Bk}$ and $A_\eta=1+\eta nB_\phi^2$. Since $\norm{\Bk}\ge1$, we have $\abs{a_{\Bk}}\le A_\eta$. Iterating \eqref{eq:point-recursion} gives
\begin{equation}\label{eq:hf-duhamel}
e_{\Bk}^{(t)}
=a_{\Bk}^t e_{\Bk}^{(0)}
-\eta\sum_{q=0}^{t-1}a_{\Bk}^{t-1-q}\mathcal M_{\Bk}e^{(q)}.
\end{equation}
The identity $a_{\Bk}^t-1=(a_{\Bk}-1)\sum_{q=0}^{t-1}a_{\Bk}^q$ and \eqref{eq:hf-diagonal} imply
\begin{equation}
\abs{a_{\Bk}^t-1}
\le\eta tA_\eta^{t-1}nB_\phi^2\norm{\Bk}^{-2p}.
\end{equation}
Using \eqref{eq:hf-mixing} in the sum in \eqref{eq:hf-duhamel}, and dividing by $\abs{f_{\Bk}}\ge c_f\norm{\Bk}^{-\alpha}$, therefore yields
\begin{equation}\label{eq:hf-complex-movement}
\begin{aligned}
\Gamma_t(\Bk)
\le\eta tA_\eta^{t-1}nB_\phi^2\norm{\Bk}^{-2p}
+\frac{\eta tA_\eta^{t-1}B_\phi\sqrt{n\lambda_1}\,E_t}{c_f}
\norm{\Bk}^{\alpha-p}.
\end{aligned}
\end{equation}
This proves \eqref{eq:linear-fp}, with the displayed coefficients as $C_{1,t}$ and $C_{2,t}$. Both powers tend to zero because $0<\alpha<p$.
\end{proof}

For the target condition \eqref{eq:hf-target-two-sided}, examples include smooth functions with unequal endpoint values, such as $f(x)=x$ ($\alpha=1$); smooth functions with matching endpoint values but unequal endpoint derivatives, such as $f(x)=x(1-x)$ ($\alpha=2$); functions with an interior jump, such as $f(x)=\mathbf 1_{(0,1/2)}(x)$ ($\alpha=1$); and functions with a weak endpoint singularity, such as $f(x)=x^{-\beta}$, $0<\beta<1/2$ ($\alpha=1-\beta$). For these four examples, one may take all sufficiently large nonzero lattice frequencies, except that the jump example requires $\mathcal A=\{2\pi m:m\in\mathbb Z\text{ is odd}\}$ to exclude zero coefficients. The condition $\alpha<\nu+1$ must also hold: the first and third examples require $\nu\ge1$, the second requires $\nu\ge2$, and the fourth is compatible with $\nu=0$.

The estimates in Theorem \ref{thm:linear-fp} should be interpreted as an asymptotic high-frequency form of the frequency principle, not as a monotone ordering of all Fourier modes. In other words, for fixed time $t$, given some frequency $\Bk_0 \in \mathcal A$ with $\Gamma_t(\Bk_0)\neq 0$, there exists $K >0$ such that $\Gamma_t(\Bk)<\Gamma_t(\Bk_0)$ for all $\norm{\Bk}>K$; namely, for all sufficiently large frequencies, their relative changes must be smaller than that of $\Bk_0$. Next, we prove a stronger result uniformly for $t$.

\begin{corollary}\label{cor:uniform-relative-movement}
Assume \eqref{eq:hf-basis-decay} and the lower bound
\begin{equation}\label{lower bound}
\abs{f_{\Bk}}\ge c_f\norm{\Bk}^{-\alpha},\qquad \Bk\in\mathcal A,
\end{equation}
where $c_f>0$ and $0<\alpha<\nu+1$. Let $h^{(0)}=0$ and $0<\eta<2/\lambda_1$. Then
\begin{equation}\label{eq:uniform-relative-movement}
\sup_{t\in\N_0}\Gamma_t(\Bk)
\le\frac{2\norm{P_Vf}_{L^2(\Omega)}B_\phi}{c_f}
\sqrt{\frac{n}{\lambda_n}}\,
\norm{\Bk}^{\alpha-(\nu+1)},
\qquad \Bk\in\mathcal A.
\end{equation}
If $0<\eta\le1/\lambda_1$, the factor $2$ can be replaced by $1$. Consequently, $\sup_{t\in\N_0}\Gamma_t(\Bk)\to0$ as $\norm{\Bk}\to\infty$.
\end{corollary}

\begin{proof}
The spectral representation in Proposition~\ref{prop:per-frequency}, with $h^{(0)}=0$, gives
\begin{equation}\label{ht}
h^{(t)}
=\sum_{j=1}^n\bigl[1-(1-\eta\lambda_j)^t\bigr]
\ipO{f}{q_j}q_j.
\end{equation}
Since $\abs{1-\eta\lambda_j}<1$, orthonormality implies
\begin{equation}
\norm{h^{(t)}}_{L^2(\Omega)}\le2\norm{P_Vf}_{L^2(\Omega)}.
\end{equation}
For $0<\eta\le1/\lambda_1$, every multiplier $1-(1-\eta\lambda_j)^t$ lies in $[0,1]$, so the factor $2$ is unnecessary. Moreover, $h^{(t)}\in V$ and \eqref{eq:capacity} give
\begin{equation}
\abs{h_{\Bk}^{(t)}}
=\abs{\ipO{h^{(t)}}{P_V\chi_{\Bk}}}
\le\norm{h^{(t)}}_{L^2(\Omega)}\ell_{\Bk}^{1/2}.
\end{equation}
By \eqref{eq:capacity-bound},
\begin{equation}
\ell_{\Bk}^{1/2}
\le B_\phi\sqrt{\frac{n}{\lambda_n}}\,
\norm{\Bk}^{-(\nu+1)}.
\end{equation}
Combining these estimates and dividing by $\abs{f_{\Bk}}\ge c_f\norm{\Bk}^{-\alpha}$ proves \eqref{eq:uniform-relative-movement} uniformly in $t$.
\end{proof}

Note that the estimate in Corollary~\ref{cor:uniform-relative-movement} holds uniformly over all $t\in\N_0$, whereas Theorem~\ref{thm:linear-fp} gives an estimate for each fixed $t$. The uniform estimate shows that $\Gamma_t(\Bk)$ remains uniformly small for $t$ at sufficiently high frequencies $\Bk\in\mathcal A$. As $t\to\infty$, the approximation converges to $P_Vf$, and the limiting ratio $\abs{(P_Vf)_{\Bk}}/\abs{f_{\Bk}}$ satisfies the same bound. The term $\sup_{t\in\N_0}\Gamma_t(\Bk)$ can be viewed as the maximal relative change of frequency $k$ throughout training. Corollary \ref{cor:uniform-relative-movement} implies that given a time $t_0\ge1$ and a frequency $\Bk_0\in\mathcal A$ with $\Gamma_{t_0}(\Bk_0)>0$, there exists $K>0$, independent of $t$, such that $\Gamma_t(\Bk)<\Gamma_{t_0}(\Bk_0)$ for all $t\in\N_0$ and all $\Bk\in\mathcal A$ with $\norm{\Bk}>K$. In other words, high-frequency errors remain nearly unchanged throughout training. 

These statements concern a fixed finite linear class. The factor $B_\phi\sqrt{n/\lambda_n}$ in \eqref{eq:uniform-relative-movement} may grow when the class is enlarged or its Gram matrix becomes ill-conditioned. The estimate is consequently not uniform over families of increasing dimension without additional assumptions.

The estimates above show that $\Gamma_t(\Bk)$ can remain small even after convergence to $P_Vf$, so small relative change alone does not necessarily imply slow learning. In the next section, we compare how quickly $h_{\Bk}^{(t)}$ approaches $(P_Vf)_{\Bk}$ at different frequencies.

\subsection{Numerical experiments}
\subsubsection{Legendre bases}
We use the orthonormal tensor product Legendre basis on $\Omega_d=[-1,1]^d$, for $d=1,2,3$,
\begin{equation}\label{eq:experiment-legendre-basis}
\phi_{\bm p}(\bm x)=\prod_{j=1}^d\sqrt{\frac{2p_j+1}{2}}P_{p_j}(x_j), \qquad \bm p\in\{0,\ldots,M_d-1\}^d,
\end{equation}
where $P_p$ is the Legendre polynomial of degree $p$. We take $(M_1,M_2,M_3)=(12,10,8)$, giving $(n_1,n_2,n_3)=(12,100,512)$ basis functions, and choose the targets
\begin{align}
f_1(x_1)&=\abs{x_1}^{-0.30},\label{eq:f-1d}\\
f_2(x_1,x_2)&=\abs{x_1}^{-0.35}\abs{x_2-0.27}^{-0.15},\label{eq:f-2d}\\
f_3(x_1,x_2,x_3)&=\abs{x_1}^{-0.40}\abs{x_2-0.27}^{-0.15}\abs{x_3+0.31}^{-0.20}.\label{eq:f-3d}
\end{align}
We apply full batch gradient descent from zero initialization with $\eta=0.01$ and record $t=1,5,20,100,500,2000$. The integrals are evaluated by 1800-point Gauss--Legendre quadrature after a change of variables removes the target singularities. Along the frequency ray
\begin{equation}\label{eq:experiment-frequency-ray}
\bm\omega_k=\bigl(\pi(k+\tfrac12),0,\ldots,0\bigr),\qquad k\ge1,
\end{equation}
we sample $k=1,\ldots,K_d$, with $(K_1,K_2,K_3)=(180,200,220)$. The low- and high-frequency bands are $k\le\lfloor0.15K_d\rfloor$ and $k\ge\lceil0.75K_d\rceil$, respectively.

For these finite Legendre classes, Lemma~\ref{lem:trace-decay} gives the basis decay in \eqref{eq:hf-basis-decay} with $\nu=0$. Along \eqref{eq:experiment-frequency-ray}, the target spectra satisfy
\begin{equation}\label{eq:experiment-target-decay}
c_d\norm{\bm\omega_k}^{-\beta_d}\le\abs{\wh{f_d}(\bm\omega_k)}\le C_d\norm{\bm\omega_k}^{-\beta_d},\qquad (\beta_1,\beta_2,\beta_3)=(0.70,0.65,0.60),
\end{equation}
for all sufficiently large $k$, with $0<c_d\le C_d$. Substituting these exponents into Theorem~\ref{thm:linear-fp} gives
\begin{equation}\label{eq:experiment-linear-prediction}
\Gamma_t^{(d)}(\bm\omega_k)\le C_{1,t}^{(d)}\norm{\bm\omega_k}^{-2}+C_{2,t}^{(d)}\norm{\bm\omega_k}^{\beta_d-1}=O_t\!\left(k^{-2}+k^{\beta_d-1}\right).
\end{equation}
Thus the predicted decay rates are $O_t(k^{-0.30})$, $O_t(k^{-0.35})$, and $O_t(k^{-0.40})$ for $d=1,2,3$, respectively.

Figure~\ref{fig:linear-legendre-gamma} shows that the relative movement increases with $t$, while its high-frequency envelope decreases with $k$. At $t=2000$, the ratios of the mean movement in the low-frequency band to that in the high-frequency band are $4.14$, $2.32$, and $6.04$ for $d=1,2,3$, respectively. These results agree with \eqref{eq:experiment-linear-prediction}: high-frequency movement is smaller on average, with local oscillations across the sampled frequencies.
\begin{figure}[ht]
\centering
\includegraphics[width=.96\textwidth]{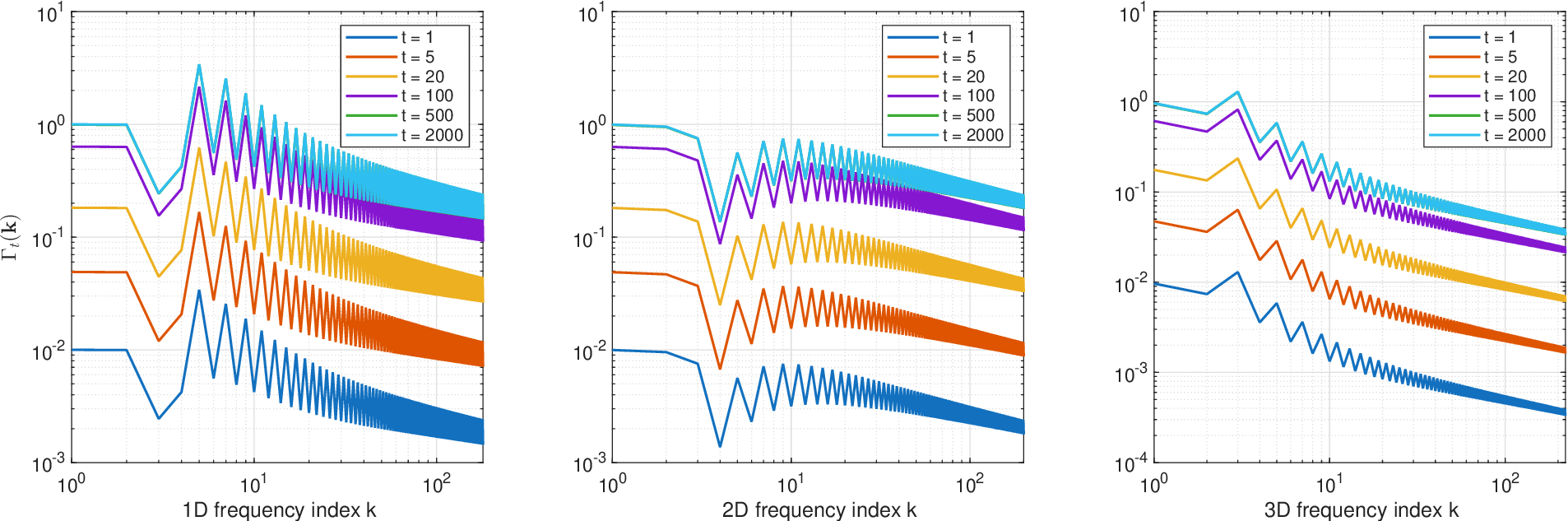}
\caption{Pointwise relative Fourier movement for the one-, two-, and three-dimensional tensor product Legendre approximation classes.}
\label{fig:linear-legendre-gamma}
\end{figure}

\subsubsection{trigonometric bases}    
We use the periodic box $\mathbb T^d=[0,2\pi]^d$, for $d=1,2,3$, and the orthonormal trigonometric basis
\begin{equation}\label{eq:experiment-trigonometric-basis}
\tilde \phi_{\bm p}(\bm x)=(2\pi)^{-d/2}\ee^{\ii\bm p\cdot\bm x},\qquad \bm p\in\mathcal I_d:=\{-20,\ldots,20\}^d.
\end{equation}
The target is
\begin{equation}
\tilde f_d=\sum_{\bm p\in\mathcal I_d}b_{\bm p}\tilde\phi_{\bm p},\qquad b_{\bm p}=(1+\norm{\bm p})^{-2}.
\end{equation}
We apply full batch gradient descent from zero initialization with $\eta=0.01$ and record $t=1,5,20,100,500,2000$. For the band averages, we use $0<r\le r_{\max}/4$ and $r\ge3r_{\max}/4$, where $r=\norm{\bm p}$ and $r_{\max}=20\sqrt d$.

Since the Gram matrix is the identity and the target belongs to the approximation class, every represented Fourier mode satisfies
\begin{equation}\label{eq:experiment-trigonometric-prediction}
e_{\bm p}^{(t)}=0.99^t e_{\bm p}^{(0)},\qquad \Gamma_t(\bm p)=1-0.99^t.
\end{equation}
At the six recorded iterations, the theoretical movement values are approximately $0.01$, $0.049010$, $0.182093$, $0.633968$, $0.993430$, and $0.9999999981$, respectively. These values are independent of frequency and dimension.

Figure~\ref{fig:linear-trigonometric-gamma} shows the computed movement averaged over unit-width radial shells. The curves are flat to machine precision in all three dimensions, with a maximum discrepancy of $1.11\times10^{-15}$ from \eqref{eq:experiment-trigonometric-prediction}. At $t=2000$, the ratio of the low- to high-frequency mean movement is $1$ in each dimension. Thus all represented frequencies have the same relative movement, and this periodic class exhibits no frequency ordering.
\begin{figure}[ht]
\centering
\includegraphics[width=.96\textwidth]{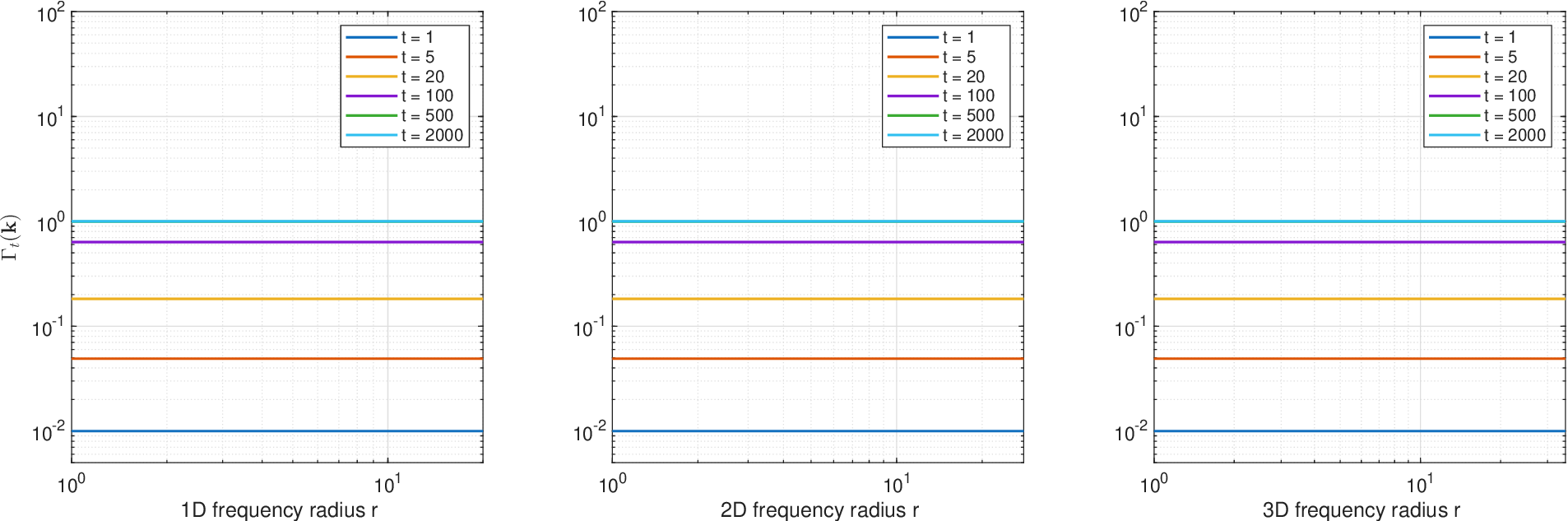}
\caption{Relative Fourier movement for the one-, two-, and three-dimensional periodic trigonometric classes.}
\label{fig:linear-trigonometric-gamma}
\end{figure} 

\section{Learning times and frequency order}\label{sec:learning}\label{sec:comparison}
Section~\ref{sec:high-frequency} bounds the output amplitudes relative to the target amplitudes at high frequencies, showing that high-frequency learner components usually change slowly under gradient descent. In this section, we derive bounds on learning times relative to the attainable limit and examine how local spectral measures and frequency mixing affect the learning order.

\subsection{Single-frequency targets}\label{sec:single-frequency}
To measure the learning progress of a single Fourier mode with $\ell_{\Bk}>0$, we first consider the single-frequency target $f=\chi_{\Bk}$. We also assume $h^{(0)}=0$, then $e_{\Bk}^{(0)}=-1$, and \eqref{eq:e-infinity} gives $e_{\Bk}^{(\infty)}=-\ipO{\chi_{\Bk}-P_V\chi_{\Bk}}{\chi_{\Bk}}=-(1-\ell_{\Bk})$. For $0<\eps<1$, we define the learning time required to reduce the relative error below $\eps$ by
\begin{equation}\label{eq:rho-T}
T_\eps(\Bk)=\min\left\{t\in\N_0:\rho_{\Bk}(t)\le\eps\right\},
\end{equation}
where $\rho_{\Bk}(t)
:=\frac{e_{\Bk}^{(t)}-e_{\Bk}^{(\infty)}}{e_{\Bk}^{(0)}-e_{\Bk}^{(\infty)}}$ is the relative error. By Proposition~\ref{prop:per-frequency} and \eqref{eq:local-measure}, the relative error $\rho_{\Bk}(t)$ takes the form
\begin{multline}\label{eq:rho-mean}
\rho_{\Bk}(t)
=\frac{\sum_{j=1}^{n}(1-\eta\lambda_j)^t\ipO{-\chi_{\Bk}}{q_j}(q_j)_{\Bk}}{-\ell_{\Bk}}\\
=\frac{1}{\ell_{\Bk}}\sum_{j=1}^{n}\abs{(q_j)_{\Bk}}^2(1-\eta\lambda_j)^t
=\sum_{j=1}^{n}\varpi_j(\Bk)(1-\eta\lambda_j)^t.
\end{multline}

Thus $\rho_{\Bk}(t)$ is the mean of the Landweber residual polynomial $(1-\eta\lambda)^t$ \cite{engl1996regularization} under the local spectral measure. For $0<\eta\le1/\lambda_{1}$, it is real and decreases monotonically from one to zero. In $\rho_{\Bk}(t)$, each decay factor $(1-\eta\lambda_j)^t$ contributes with weight $\varpi_j(\Bk)$. By Theorem \ref{thm:linear-convergence}, the gradient descent converges, and $e_{\Bk}^{(\infty)}$ characterizes the approximation error between the target $f$ and the approximation class $V$. Specifically, $e_{\Bk}^{(\infty)}=0$ if $f\in V$. So, in the numerator of $\rho_{\Bk}(t)$, subtracting $e_{\Bk}^{(\infty)}$ removes the approximation error, and therefore $\rho_{\Bk}(t)$ measures how quickly the attainable limit is reached.

For a single-frequency target, $\Gamma_t(\Bk)$ and $\rho_{\Bk}(t)$ satisfy an exact relation. If $0<\eta\le1/\lambda_1$, then
\begin{equation}\label{eq:movement-capacity-speed}
e_{\Bk}^{(t)}=-\bigl(1-\ell_{\Bk}+\ell_{\Bk}\rho_{\Bk}(t)\bigr),
\qquad
\Gamma_t(\Bk)=\ell_{\Bk}\bigl(1-\rho_{\Bk}(t)\bigr).
\end{equation}
Indeed, the first identity follows from $e_{\Bk}^{(0)}=-1$ and $e_{\Bk}^{(\infty)}=-(1-\ell_{\Bk})$. The second follows from $f_{\Bk}=1$ and $h_{\Bk}^{(t)}=e_{\Bk}^{(t)}+1=\ell_{\Bk}(1-\rho_{\Bk}(t))\ge0$, since $0\le\rho_{\Bk}(t)\le1$. Therefore, small capacity $\ell_{\Bk}$ can keep the relative change $\Gamma_t(\Bk)$ small throughout training, even when the represented part is learned quickly (i.e., $\rho_{\Bk}(t)\rightarrow0$ quickly).

The next theorem gives two-sided learning-time bounds in terms of the local spectral measure. 

\begin{theorem}\label{thm:learning-time}
Let $0<\eta\le1/\lambda_{1}$, $\Bk\in\Lambda$ with $\ell_{\Bk}>0$, and $0<\eps<1$. Then the following statements hold.
\begin{enumerate}
\item[(a)] The learning time satisfies
\begin{equation}\label{eq:jensen-time}
T_\eps(\Bk)\ge\frac{\log(1/\eps)}{-\log(1-\eta\bar\lambda_{\Bk})}\ge \left(\frac{1}{\eta\bar\lambda_{\Bk}}-1\right)\log(1/\eps).
\end{equation}
\item[(b)] For $0<\mu\le\lambda_{1}$: \\
if $\sum_{\lambda_j\le\mu}\varpi_j(\Bk)>\eps$, then
\begin{equation}\label{eq:quantile-lower}
T_\eps(\Bk)\ge\frac{\log\bigl(\eps^{-1}\sum_{\lambda_j\le\mu}\varpi_j(\Bk)\bigr)}{-\log(1-\eta\mu)};
\end{equation}
if $\sum_{\lambda_j<\mu}\varpi_j(\Bk)<\eps$, then
\begin{equation}\label{eq:harmonic-time}
T_\eps(\Bk)\le1+\frac{\sum_{\lambda_j\ge\mu}\varpi_j(\Bk)/\lambda_j}{\ee\,\eta\,\bigl(\eps-\sum_{\lambda_j<\mu}\varpi_j(\Bk)\bigr)},
\end{equation}
and moreover, if $\sum_{\lambda_j<\mu}\varpi_j(\Bk)\le c\eps$ with $0\le c<1$, then
\begin{equation}\label{eq:quantile-time}
T_\eps(\Bk)\le1+\frac{\log\bigl(1/((1-c)\eps)\bigr)}{\eta\mu}.
\end{equation}
\end{enumerate}
The lower bounds for $T_\eps(\Bk)$ in (a) and (b) are understood as zero if $\eta\bar\lambda_{\Bk}=1$ or $\eta\mu=1$, respectively.
\end{theorem}

\begin{proof}
(a) Since $\rho_{\Bk}(0)=1>\eps$, we have $T_\eps(\Bk)\ge1$. For each integer $t\ge1$, the function $g_t(\lambda)=(1-\eta\lambda)^t$ is convex on $[0,1/\eta]$. It is affine when $t=1$; for $t\ge2$, its second derivative is $t(t-1)\eta^2(1-\eta\lambda)^{t-2}\ge0$ on the interior, and it extends continuously to the endpoints. All eigenvalues $\lambda_j$ lie in $(0,\lambda_{1}]\subset(0,1/\eta]$. Since the weights $\varpi_j(\Bk)$ are nonnegative and sum to one, Jensen's inequality and \eqref{eq:speed} give
\begin{equation}
\rho_{\Bk}(t)=\sum_{j=1}^{n}\varpi_j(\Bk)g_t(\lambda_j)\ge g_t(\bar\lambda_{\Bk})=(1-\eta\bar\lambda_{\Bk})^t,\qquad t\ge1.
\end{equation}
If $\eta\bar\lambda_{\Bk}<1$, evaluating this inequality at $t=T_\eps(\Bk)$ gives $(1-\eta\bar\lambda_{\Bk})^{T_\eps(\Bk)}\le\eps$. Taking logarithms proves the first inequality in \eqref{eq:jensen-time}. The second follows from $-\log(1-z)\le z/(1-z)$ for $0\le z<1$. If $\eta\bar\lambda_{\Bk}=1$, both lower bounds in \eqref{eq:jensen-time} are zero under the stated convention, so the conclusion also holds.

(b) Fix $0<\mu\le\lambda_{1}$ and an integer $t\ge1$. Every term in \eqref{eq:rho-mean} is nonnegative. For $\lambda_j\le\mu$, we have $(1-\eta\lambda_j)^t\ge(1-\eta\mu)^t$, which gives a lower bound for $\rho_{\Bk}(t)$ after dropping the terms with $\lambda_j>\mu$. For an upper bound, we use $(1-\eta\lambda_j)^t\le1$ when $\lambda_j<\mu$ and $(1-\eta\lambda_j)^t\le(1-\eta\mu)^t$ when $\lambda_j\ge\mu$, together with $\sum_{\lambda_j\ge\mu}\varpi_j(\Bk)\le1$. Thus
\begin{equation}\label{eq:quantile}
\biggl(\sum_{\lambda_j\le\mu}\varpi_j(\Bk)\biggr)(1-\eta\mu)^t\le\rho_{\Bk}(t)\le\sum_{\lambda_j<\mu}\varpi_j(\Bk)+(1-\eta\mu)^t.
\end{equation}
If $\sum_{\lambda_j\le\mu}\varpi_j(\Bk)>\eps$ and $\eta\mu<1$, evaluating the lower bound at $t=T_\eps(\Bk)$ and using $\rho_{\Bk}(T_\eps(\Bk))\le\eps$ gives \eqref{eq:quantile-lower} after taking logarithms. If $\eta\mu=1$, the stated lower bound is zero, so \eqref{eq:quantile-lower} also holds.

To prove \eqref{eq:harmonic-time}, let $t\ge1$. For $0\le z\le1$, we have $z(1-z)^t\le t^t/(t+1)^{t+1}\le1/(\ee t)$, since $(1+1/t)^{t+1}\ge\ee$. With $z=\eta\lambda_j$, this gives $(1-\eta\lambda_j)^t\le1/(\ee\eta t\lambda_j)$. Applying this estimate to the terms with $\lambda_j\ge\mu$ and using $(1-\eta\lambda_j)^t\le1$ for the remaining terms yields
\begin{equation}\label{eq:harmonic}
\rho_{\Bk}(t)\le\sum_{\lambda_j<\mu}\varpi_j(\Bk)+\frac{1}{\ee\eta t}\sum_{\lambda_j\ge\mu}\frac{\varpi_j(\Bk)}{\lambda_j}.
\end{equation}
If $\sum_{\lambda_j<\mu}\varpi_j(\Bk)<\eps$, this implies $\rho_{\Bk}(t)\le\eps$ for every integer $t\ge1$ with
\begin{equation}
t\ge\frac{\sum_{\lambda_j\ge\mu}\varpi_j(\Bk)/\lambda_j}{\ee\,\eta\,\bigl(\eps-\sum_{\lambda_j<\mu}\varpi_j(\Bk)\bigr)}.
\end{equation}
The right-hand side is positive. Rounding it up to an integer increases it by at most one, which proves \eqref{eq:harmonic-time}.

Finally, let $0\le c<1$ and suppose that $\sum_{\lambda_j<\mu}\varpi_j(\Bk)\le c\eps$. The upper bound in \eqref{eq:quantile} gives
\begin{equation}
\rho_{\Bk}(t)\le c\eps+(1-\eta\mu)^t\le c\eps+\ee^{-\eta\mu t},\qquad t\ge1.
\end{equation}
Thus $\rho_{\Bk}(t)\le\eps$ whenever $\ee^{-\eta\mu t}\le(1-c)\eps$, which holds for every integer $t\ge\log\bigl(1/((1-c)\eps)\bigr)/(\eta\mu)$. Since $0<(1-c)\eps<1$, this threshold is positive. Rounding it up to an integer proves \eqref{eq:quantile-time}.
\end{proof}

This theorem provides upper and lower bounds for $T_\eps(\Bk)$ under certain conditions. Note that $\bar\lambda_{\Bk}\leq\lambda_{1}$, Theorem \ref{thm:learning-time}(a) means that the required time to make the error $\rho_{\Bk}(t)$ reach $\eps$ is at least $\left(\frac{1}{\eta\lambda_{1}}-1\right)\log(1/\eps)$.  In Theorem \ref{thm:learning-time}(b), the sum $\sum_{\lambda_j\le\mu}\varpi_j(\Bk)$ represents the proportion of $P_V\chi_{\Bk}$ in directions with eigenvalues at most $\mu$.
If this proportion is $\Theta(1)$ (greater than $\eps$), then it has to cost $\Theta\left(-\log\bigl(\eps^{-1}\bigr)/{\log(1-\eta\mu)}\right)$ time to make the error reach $\eps$; if $\sum_{\lambda_j<\mu}\varpi_j(\Bk)<\Theta(\eps)$ with constant less than 1, reducing the error below $\eps$ requires $\min \{\Theta\left(\frac{1}{\eta\epsilon}\right), \Theta\left(\frac{\log(1/\epsilon)}{\eta\mu}\right)\}$ time. Here we write $A=\Theta(B)$ if there exist constants $0<c_1\le c_2$, independent of the varying parameters, such that $c_1B\le A\le c_2B$.

Fix $0\le c<1$. Let $U_\eps(\Bk)$ be the minimum of the upper bounds \eqref{eq:quantile-time} and \eqref{eq:harmonic-time} over $\mu\in(0,\lambda_{1}]$ satisfying $\sum_{\lambda_j<\mu}\varpi_j(\Bk)\le c\eps$ and $\sum_{\lambda_j<\mu}\varpi_j(\Bk)<\eps$, respectively. Similarly, let $L_\eps(\Bk)$ be the maximum of the lower bound \eqref{eq:jensen-time} and the lower bounds \eqref{eq:quantile-lower} over $\mu\in(0,\lambda_{1}]$ satisfying $\sum_{\lambda_j\le\mu}\varpi_j(\Bk)>\eps$. Theorem~\ref{thm:learning-time} gives
\begin{equation}\label{eq:U-L}
L_\eps(\Bk)\le T_\eps(\Bk)\le U_\eps(\Bk).
\end{equation}

By linear independence, it is trivial that if the represented part $P_V\chi_{\Bk}$ is learned accurately, i.e., $h_{\bm c_{\Bk}}=P_V\chi_{\Bk}$, then the unique coefficient vector is $\bm c_{\Bk}=\sum_{j=1}^n\lambda_j^{-1/2}\,\overline{(q_j)_{\Bk}}\,\bm v_j$. Its norm satisfies $\norm{\bm c_{\Bk}}^2=\ell_{\Bk}\sum_{j=1}^n\varpi_j(\Bk)/\lambda_j$. We define the harmonic mean $\lambda^H_{\Bk}=(\sum_{j=1}^{n}\varpi_j(\Bk)/\lambda_j)^{-1}=\ell_{\Bk}/\norm{\bm c_{\Bk}}^2$. 

Now we compare the harmonic mean $\lambda^H_k$ and the arithmetic mean $\bar\lambda_{k'}$. Specifically, on $\Omega=[-1,1]$, we take
\begin{equation}\label{eq:two-polynomial-example}
\phi_1(x)=\sqrt{\frac32}\,x,
\qquad
\phi_2(x)=\frac{\sqrt{\delta}}{N}(x^2-d_0),
\end{equation}
where
\begin{equation}
\delta=10^{-6},
\qquad
d_0=1-\frac{8}{\pi^2},
\qquad
N=\left(\frac25-\frac{4d_0}{3}+2d_0^2\right)^{1/2}.
\end{equation}
The two basis functions are orthogonal, and the Gram eigenvalues are $1$ and $\delta$. Figure~\ref{fig:two-polynomial-spectral-means} shows $\lambda^H_k$ at $k=\pi(m+1/2)$, $m=0,\ldots,40$, and $\bar\lambda_{k'}$ at $k'=\pi(m'+1/2)$, $m'=0,\ldots,1000$. Both means decrease strictly along the positive half-integer lattice and converge to $\delta$. The harmonic mean is already close to $\delta$ at $k=3\pi/2$, whereas the arithmetic mean approaches this limit more gradually.

\begin{figure}[!ht]
\centering
\includegraphics[width=.49\textwidth]{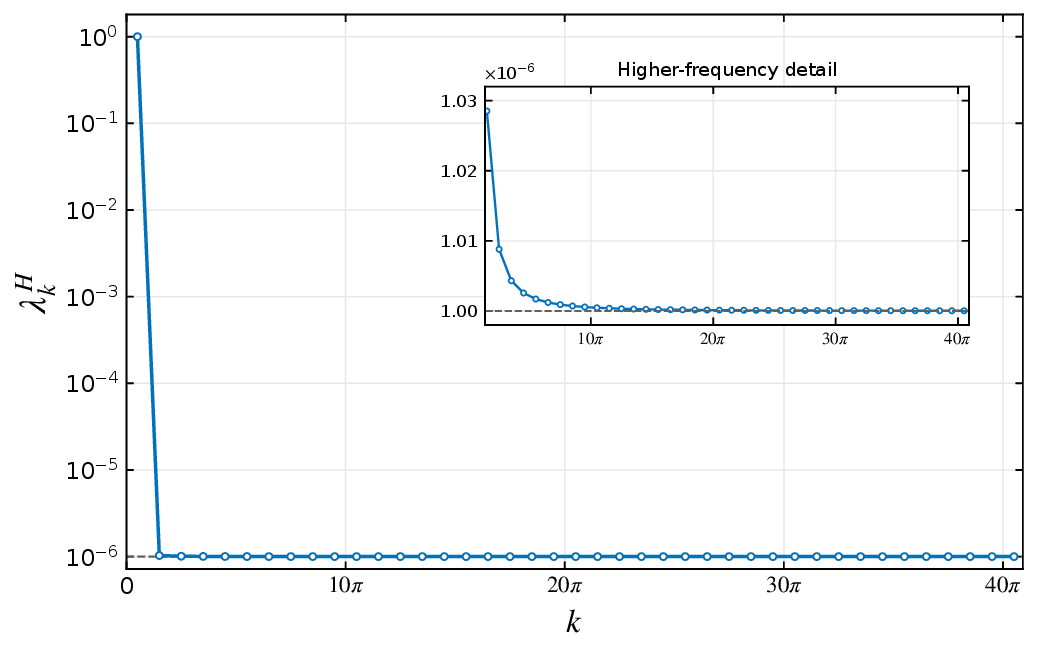}\hfill
\includegraphics[width=.49\textwidth]{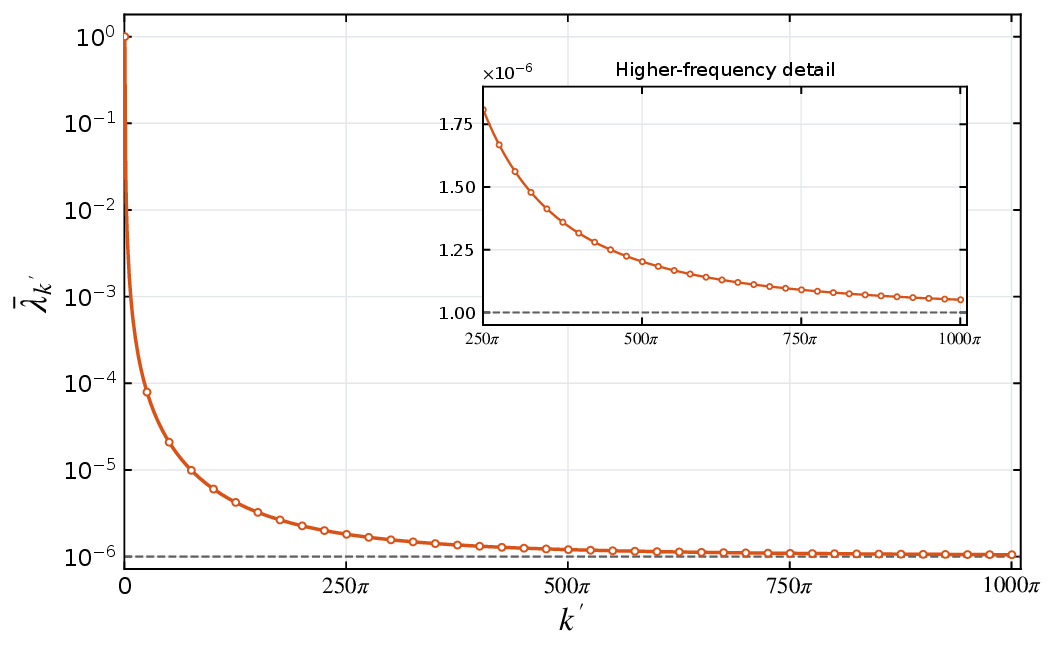}
\caption{Comparison between $\lambda^H_k$ (left) and $\bar\lambda_{k'}$ (right) for the two-polynomial basis.}
\label{fig:two-polynomial-spectral-means}
\end{figure}

The next theorem compares two separate training runs, one with target $\chi_{\Bk}$ and one with target $\chi_{\Bk'}$, using the same basis functions, learning rate, and zero initialization.

\begin{theorem}\label{thm:comparison}
Let $0<\eta\le1/\lambda_{1}$, $0<\eps<1$, and $0\le c<1$, and let $\Bk,\Bk'\in\Lambda$ with $\ell_{\Bk}>0$ and $\ell_{\Bk'}>0$.
\begin{enumerate}
\item[(a)] If $U_\eps(\Bk)<L_\eps(\Bk')$, then $T_\eps(\Bk)<T_\eps(\Bk')$. 
\item[(b)] Assume that there are $0<\mu'<\mu\le\lambda_{1}$ and $\eps<p\le1$ such that $\sum_{\lambda_j<\mu}\varpi_j(\Bk)\le c\eps$, $\sum_{\lambda_j\le\mu'}\varpi_j(\Bk')\ge p$, and
\begin{equation}\label{eq:separation}
\frac{\mu'}{\mu}<\frac{\log(p/\eps)}{\log(p/\eps)+\log\bigl(1/((1-c)\eps)\bigr)+1}.
\end{equation}
Then $U_\eps(\Bk)<L_\eps(\Bk')$, and hence $T_\eps(\Bk)<T_\eps(\Bk')$.
\item[(c)] If
\begin{equation}\label{eq:mean-separation}
\bar\lambda_{\Bk'}<\kappa(\eps)\,\lambda^H_{\Bk},\qquad \kappa(\eps)=\frac{\ee\eps\log(1/\eps)}{1+\ee\eps\,(1+\log(1/\eps))},
\end{equation}
then $U_\eps(\Bk)<L_\eps(\Bk')$, and hence $T_\eps(\Bk)<T_\eps(\Bk')$.
\end{enumerate}
\end{theorem}

\begin{proof}
(a) By \eqref{eq:U-L}, $T_\eps(\Bk)\le U_\eps(\Bk)<L_\eps(\Bk')\le T_\eps(\Bk')$.

(b) By \eqref{eq:quantile-time}, $U_\eps(\Bk)\le\log\bigl(1/((1-c)\eps)\bigr)/(\eta\mu)+1$. By \eqref{eq:quantile-lower}, $\sum_{\lambda_j\le\mu'}\varpi_j(\Bk')\ge p>\eps$, and $-\log(1-z)\le z/(1-z)$, we have
\begin{equation}
L_\eps(\Bk')\ge\frac{\log(p/\eps)}{-\log(1-\eta\mu')}\ge\frac{(1-\eta\mu')\log(p/\eps)}{\eta\mu'}.
\end{equation}
Since $\eta\mu\le1$, we have $\eta\mu'\le\mu'/\mu$. Hence
\begin{equation}
\begin{aligned}
\eta\mu'\bigl(L_\eps(\Bk')-U_\eps(\Bk)\bigr)&\ge(1-\eta\mu')\log(p/\eps)-\frac{\mu'}{\mu}\log\bigl(1/((1-c)\eps)\bigr)-\eta\mu'\\
&\ge\Bigl(1-\frac{\mu'}{\mu}\Bigr)\log(p/\eps)-\frac{\mu'}{\mu}\Bigl(\log\bigl(1/((1-c)\eps)\bigr)+1\Bigr),
\end{aligned}
\end{equation}
and the right-hand side is positive by \eqref{eq:separation}.

(c) Take $\mu=\lambda_{n}$ in \eqref{eq:harmonic-time}. Since $\sum_{\lambda_j<\lambda_{n}}\varpi_j(\Bk)=0$ and $\sum_{j=1}^{n}\varpi_j(\Bk)/\lambda_j=1/\lambda^H_{\Bk}$, we have $U_\eps(\Bk)\le1+1/(\ee\eta\eps\lambda^H_{\Bk})$. By \eqref{eq:jensen-time}, $L_\eps(\Bk')\ge(1-\eta\bar\lambda_{\Bk'})\log(1/\eps)/(\eta\bar\lambda_{\Bk'})$. Let $x=\eta\bar\lambda_{\Bk'}$ and $y=\eta\lambda^H_{\Bk}$. Then $y\le1$, since $\lambda^H_{\Bk}\le\lambda_{1}$, and $x<\kappa(\eps)y<1$. Using $x\le x/y$, we obtain
\begin{equation}
x\bigl(L_\eps(\Bk')-U_\eps(\Bk)\bigr)\ge(1-x)\log(1/\eps)-x-\frac{x}{\ee\eps y}\ge\log(1/\eps)-\frac{x}{y}\Bigl(\log(1/\eps)+1+\frac{1}{\ee\eps}\Bigr).
\end{equation}
Since $\kappa(\eps)=\log(1/\eps)/(\log(1/\eps)+1+1/(\ee\eps))$, the right-hand side is positive by \eqref{eq:mean-separation}.
\end{proof}

Theorem~\ref{thm:comparison} gives sufficient conditions for the learning order in these two single-frequency runs. Result (a) gives the general form; results (b) and (c) give sufficient conditions for (a). We remark that the conditions in (b) and (c) are not strict, which can be satisfied by many basis classes, e.g., unnormalized Legendre bases and weighted Fourier bases for suitable frequency pairs. See Section~\ref{sec:exp-comparison} for numerical examples.

We usually study the frequency principle by comparing two frequencies with $\norm{\Bk}<\norm{\Bk'}$ and asking whether $T_\eps(\Bk)<T_\eps(\Bk')$. Here Theorem~\ref{thm:comparison}(c) uses the condition $\bar\lambda_{\Bk'}<\kappa(\eps)\lambda^H_{\Bk}$ as a sufficient criterion in place of $k<k'$. In general, $\bar\lambda_{\Bk'}$ does not tend to zero; instead, it approaches a positive limit. Thus taking $k'$ sufficiently large does not guarantee \eqref{eq:mean-separation} for every fixed $k$ and $\eps$. For suitable choices of $k$ and $\eps$, however, the condition can hold for all sufficiently large $k'$. In this case, $T_\eps(k)<T_\eps(k')$, giving a limited form of the frequency principle. 

For many classes, the condition in Theorem~\ref{thm:comparison}(a) also implies the frequency order. For example, take $\phi_s=(1+\norm{\Bk_s})^{-p}\chi_{\Bk_s}$, $s=1,\ldots,n$, where $p>0$ and $\Bk_1,\ldots,\Bk_n\in\Lambda$ are distinct. Under the assumptions of Theorem~\ref{thm:comparison}, each selected frequency $\Bk$ has $\ell_{\Bk}=1$ and $\rho_{\Bk}(t)=\bigl[1-\eta(1+\norm{\Bk})^{-2p}\bigr]^t$. Hence $\rho_{\Bk}(t)$ and $T_\eps(\Bk)$ are nondecreasing in $\norm{\Bk}$. Consequently, for $\Bk,\Bk'\in\{\Bk_1,\ldots,\Bk_n\}$, if $U_\eps(\Bk)<L_\eps(\Bk')$, then $\norm{\Bk}<\norm{\Bk'}$. This example clarifies the general frequency principle. The usual low-frequency priority states that $T_\eps(\Bk)<T_\eps(\Bk')$ whenever $\norm{\Bk}<\norm{\Bk'}$. In our general formulation, the condition $U_\eps(\Bk)<L_\eps(\Bk')$ guarantees $T_\eps(\Bk)<T_\eps(\Bk')$ under the assumptions of Theorem~\ref{thm:comparison}. Under additional conditions on the basis, as in the example above, the same comparison also implies $\norm{\Bk}<\norm{\Bk'}$. Thus, the comparison condition determines both the learning-time order and the frequency order in these classes.

The following corollary connects this condition to frequency magnitude through the decay of the basis functions.

\begin{corollary}\label{cor:decay-comparison}
Let $0<\eta\le1/\lambda_{1}$ and $0<\eps<1$, and assume \eqref{eq:basis-decay}. Let $\Bk\in\Lambda$ with $\ell_{\Bk}>0$, and let $\bm c_{\Bk}$ be the unique coefficient vector satisfying $h_{\bm c_{\Bk}}=P_V\chi_{\Bk}$. If $\Bk'\in\mathcal A$ satisfies
\begin{equation}\label{eq:decay-comparison}
\ell_{\Bk'}\norm{\Bk'}^{2(\nu+1)}>\frac{nB_\phi^2\,\norm{\bm c_{\Bk}}^2}{\kappa(\eps)\,\ell_{\Bk}},
\end{equation}
then $T_\eps(\Bk)<T_\eps(\Bk')$.
\end{corollary}

\begin{proof}
By \eqref{eq:speed} and \eqref{eq:capacity-bound}, $\bar\lambda_{\Bk'}=\mathcal K_{\Bk'\Bk'}/\ell_{\Bk'}\le nB_\phi^2\norm{\Bk'}^{-2(\nu+1)}/\ell_{\Bk'}$. By \eqref{eq:decay-comparison} and $\lambda^H_{\Bk}=\ell_{\Bk}/\norm{\bm c_{\Bk}}^2$, this is smaller than $\kappa(\eps)\lambda^H_{\Bk}$. Theorem~\ref{thm:comparison}(c) gives the claim.
\end{proof}

Corollary~\ref{cor:decay-comparison} gives a sufficient condition in terms of frequency for lower frequencies to converge faster than higher frequencies. For a fixed $\Bk$, condition \eqref{eq:decay-comparison} requires $\ell_{\Bk'}\norm{\Bk'}^{2(\nu+1)}$ to be sufficiently large relative to $\norm{\bm c_{\Bk}}^2/\ell_{\Bk}$. Thus a larger $\norm{\Bk'}$ helps satisfy the condition, provided that $\ell_{\Bk'}$ does not become too small. For a pair with $\norm{\Bk}<\norm{\Bk'}$ satisfying this condition, the lower frequency reaches the prescribed relative error first.

\subsection{General targets}\label{sec:general-targets}
Theorem~\ref{thm:comparison} compares two separate training runs, one with target $f=\chi_{\Bk}$ and the other with target $f=\chi_{\Bk'}$. We now consider a single training run whose target contains several frequencies. 

We first extend \eqref{eq:rho-mean} from the single-frequency target $\chi_{\Bk}$ to any function $f\in L^2(\Omega)$. Note that $P_Vf$ is the represented part that can be actually learned. We assume that $\ell_f:=\norm{P_Vf}_{L^2(\Omega)}^2>0$. We also define
$\varpi_j(f)=\frac{\abs{\ipO{q_j}{f}}^2}{\ell_f}=\frac{\abs{\ipO{q_j}{P_Vf}}^2}{\ell_f}$, which is the square of the ratio between the coefficient of $P_Vf$ in terms of $q_j$ and its length. Hence, $\varpi_j(f)\in[0,1]$ and $\sum_{j}\varpi_j(f)=1$.

Here we still consider zero initialization $h^{(0)}=0$. We write error $e_f^{(t)}=\ipO{h^{(t)}-f}{f}$ and $e_f^{(\infty)}=\ipO{P_Vf-f}{f}$. Then define the relative error
\begin{equation}
\rho_f(t):=\frac{e_f^{(t)}-e_f^{(\infty)}}
{e_f^{(0)}-e_f^{(\infty)}}=\frac{\langle h^{(t)},f\rangle_\Omega-\ell_f}{-\ell_f}
=1-\frac{\langle h^{(t)},f\rangle_\Omega}{\ell_f}.
\end{equation}
By \eqref{ht}, we have
\begin{equation}\label{eq:rho-direction}
\rho_f(t)=1-\frac{\ipO{h^{(t)}}{f}}{\ell_f}=\sum_{j=1}^{n}\varpi_j(f)(1-\eta\lambda_j)^t.
\end{equation}
For $0<\eta\le1/\lambda_{1}$, $\rho_f(t)$ is real and decreases monotonically from one to zero.
Similar to the notations in previous sections, we define the arithmetic mean by
\begin{equation}
\bar\lambda_{f}=\sum_{j=1}^{n}\varpi_j(f)\lambda_j,
\end{equation}
and the learning time required to reduce the relative error below $\eps$ by
\begin{equation}\label{eq:rho-T-general}
T_\eps(f)=\min\left\{t\in\N_0:\rho_{f}(t)\le\eps\right\}.
\end{equation}
Let $0<\eta\le1/\lambda_1$ and $0<\eps<1$, and fix $0\le c<1$. Define $U_\eps(f)$ as the minimum of
\begin{equation}
1+\frac{\log\bigl(1/((1-c)\eps)\bigr)}{\eta\mu}
\qquad\text{and}\qquad
1+\frac{\sum_{\lambda_j\ge\mu}\varpi_j(f)/\lambda_j}
{\ee\,\eta\,\bigl(\eps-\sum_{\lambda_j<\mu}\varpi_j(f)\bigr)}
\end{equation}
over $\mu\in(0,\lambda_1]$ satisfying $\sum_{\lambda_j<\mu}\varpi_j(f)\le c\eps$ and $\sum_{\lambda_j<\mu}\varpi_j(f)<\eps$, respectively. Define $L_\eps(f)$ as the maximum of
\begin{equation}
\frac{\log(1/\eps)}{-\log(1-\eta\bar\lambda_f)}
\qquad\text{and}\qquad
\frac{\log\bigl(\eps^{-1}\sum_{\lambda_j\le\mu}\varpi_j(f)\bigr)}
{-\log(1-\eta\mu)},
\end{equation}
where the second expression is maximized over $\mu\in(0,\lambda_1]$ satisfying $\sum_{\lambda_j\le\mu}\varpi_j(f)>\eps$. The two expressions defining $L_\eps(f)$ are taken to be zero if $\eta\bar\lambda_f=1$ or $\eta\mu=1$, respectively.

For a general target $f$, the next theorem bounds the absolute error.

\begin{theorem}\label{prop:several}
Let $0<\eta<2/\lambda_{1}$, $f\in L^2(\Omega)$, and $h^{(0)}=0$. The following statements hold for all $t\in\N_0$.
\begin{enumerate}
\item[(a)] For every $f$ with $\norm{f}_{L^2(\Omega)}=1$,
\begin{equation}\label{eq:general-target}
\abs{e_f^{(t)}-e_f^{(\infty)}}\le\norm{P_Vf}_{L^2(\Omega)}\bigl(\ell_f\,\rho_f(2t)\bigr)^{1/2}.
\end{equation}
\item[(b)] If $f=\sum_{l=1}^Ka_lu_l$, where $K\in\mathbb{Z}^+$, $a_l\in\C$ and $u_1,\ldots,u_K$ are orthonormal functions under $\ipO{\cdot}{\cdot}$ with $\ell_{u_l}>0$ ($l=1,\dots,K$), then
\begin{equation}\label{eq:several-modes}
\begin{aligned}
e_{u_l}^{(t)}-e_{u_l}^{(\infty)}&=-a_l\,\ell_{u_l}\rho_{u_l}(t)-\sum_{i\ne l}a_i\,C_{il}(t),\\
C_{il}(t)&=\sum_{j=1}^{n}(1-\eta\lambda_j)^t\ipO{u_i}{q_j}\ipO{q_j}{u_l}.
\end{aligned}
\end{equation}
The cross terms satisfy
\begin{equation}\label{eq:cross-bound}
\begin{aligned}
\abs{C_{il}(t)}&\le(\ell_{u_i}\ell_{u_l})^{1/2}\min\{\rho_{u_i}(2t),\rho_{u_l}(2t)\}^{1/2},\\
\abs{C_{il}(0)}&\le\bigl((1-\ell_{u_i})(1-\ell_{u_l})\bigr)^{1/2}.
\end{aligned}
\end{equation}
If $\eta\le1/\lambda_{1}$, then moreover $\abs{C_{il}(t)-C_{il}(0)}\le(\ell_{u_i}\ell_{u_l})^{1/2}\min\{1-\rho_{u_i}(t),1-\rho_{u_l}(t)\}^{1/2}$.
\end{enumerate}
\end{theorem}

\begin{proof}
Expanding $h^{(t)}-P_Vf$ in the basis $q_j$ as in Proposition~\ref{prop:per-frequency} gives
\begin{equation}\label{eq:direction-error}
e_f^{(t)}-e_f^{(\infty)}=\ipO{h^{(t)}-P_Vf}{f}=\sum_{j=1}^{n}(1-\eta\lambda_j)^t\ipO{h^{(0)}-f}{q_j}\ipO{q_j}{f}.
\end{equation}
Part (a) follows by the Cauchy--Schwarz inequality. Substituting $f=\sum_la_lu_l$ and $h^{(0)}=0$ gives \eqref{eq:several-modes}. Applying the same inequality to $C_{il}(t)$, with the factor $(1-\eta\lambda_j)^t$ in either sum, gives the first bound in \eqref{eq:cross-bound}. The second follows from $C_{il}(0)=-\ipO{(I-P_V)u_l}{(I-P_V)u_i}$ for $i\ne l$. If $\eta\le1/\lambda_{1}$, then $(1-(1-\eta\lambda_j)^t)^2\le1-(1-\eta\lambda_j)^t$. Another application of the Cauchy--Schwarz inequality therefore gives $\abs{C_{il}(t)-C_{il}(0)}\le\bigl(\ell_{u_i}\ell_{u_l}(1-\rho_{u_l}(t))\bigr)^{1/2}$, and interchanging $i$ and $l$ completes the proof.
\end{proof}

Let $f=au+a'u'$ with $a,a'\ne0$, where $u,u'\in L^2(\Omega)$ are orthonormal with $\ell_u,\ell_{u'}>0$. Define
\begin{equation}
\beta=\frac{\abs{a'}\,\ell_{u'}^{1/2}}{\abs{a}\,\ell_u^{1/2}},\qquad \gamma=\Bigl(\frac{(1-\ell_u)(1-\ell_{u'})}{\ell_u\ell_{u'}}\Bigr)^{1/2}.
\end{equation}
For $v\in\{u,u'\}$ with $e_v^{(0)}\ne e_v^{(\infty)}$, also define $R_v(t)=\abs{e_v^{(t)}-e_v^{(\infty)}}/\abs{e_v^{(0)}-e_v^{(\infty)}}$. For $0<\eps<1$, let $T_\eps(v;f)=\min\{t\in\N_0:R_v(s)\le\eps\text{ for all }s\ge t\}$. This is the first iteration from which the relative error stays at most $\eps$; the error need not decrease monotonically. The next theorem combines these estimates with the amplitude ratio to compare two frequencies in one run. 

\begin{theorem}\label{cor:two-modes}
For the target $f$ defined above, assume $0<\eta\le1/\lambda_{1}$, $0<\eps<1$, $h^{(0)}=0$, $\beta\gamma<1$, and $\gamma<\beta$. Then, for all $t\in\N_0$,
\begin{align}
R_u(t)&\le\frac{\rho_u(t)+\beta\rho_u(2t)^{1/2}}{1-\beta\gamma},\label{eq:two-low}\\
R_{u'}(t)&\ge\frac{\rho_{u'}(t)-\beta^{-1}\min\bigl\{\rho_u(2t)^{1/2},\,\gamma+(1-\rho_{u'}(t))^{1/2}\bigr\}}{1+\beta^{-1}\gamma}.\label{eq:two-high}
\end{align}
In particular, if the right-hand side of \eqref{eq:two-low} is at most $\eps$ and the right-hand side of \eqref{eq:two-high} is larger than $\eps$ for some $t_1\in\N_0$, then $T_\eps(u;f)\le t_1<T_\eps(u';f)$.
\end{theorem}

\begin{proof}
By Theorem~\ref{prop:several}(b) at $t=0$, we have $\abs{e_u^{(0)}-e_u^{(\infty)}}\ge\abs{a}\ell_u(1-\beta\gamma)$ and $\abs{a'}\ell_{u'}(1-\beta^{-1}\gamma)\le\abs{e_{u'}^{(0)}-e_{u'}^{(\infty)}}\le\abs{a'}\ell_{u'}(1+\beta^{-1}\gamma)$. Thus both denominators are positive. Applying \eqref{eq:cross-bound} to the two terms in \eqref{eq:several-modes} gives \eqref{eq:two-low}. For \eqref{eq:two-high}, also use $\abs{C_{12}(t)}\le\abs{C_{12}(0)}+\abs{C_{12}(t)-C_{12}(0)}$ and the last bound of Proposition~\ref{prop:several}(b), followed by the reverse triangle inequality. Finally, the upper bound in \eqref{eq:two-low} is nonincreasing. Hence $R_u(s)\le\eps$ for all $s\ge t_1$, while $R_{u'}(t_1)>\eps$ gives $T_\eps(u';f)>t_1$.
\end{proof}

When $\ell_u$ and $\ell_{u'}$ are close to one, $\gamma$ is small. If $|a|$ and $|a'|$ are also comparable, then $\beta$ is of order one. In this case, $\beta\gamma<1$ and $\gamma<\beta$ are satisfied. A specific example will be presented in Section \ref{sec:exp-general-targets}. Theorem~\ref{thm:comparison} compares two separate training runs, whereas Theorem~\ref{cor:two-modes} compares two components learned in the same run. In the latter case, each component can affect the error in the other direction through the cross terms. The conditions in Theorem~\ref{cor:two-modes} also control these effects to establish the learning order. A large difference between $|a|$ and $|a'|$ can change this order.

Here, $T_\eps(v;f)$ is the smallest $t$ such that $R_v(s)\le\eps$ for every $s\ge t$. Thus, $T_\eps(u;f)<T_\eps(u';f)$ means that the error in direction $u$ stays below $\eps$ earlier. It does not mean that this error first falls below $\eps$ earlier: the error in direction $u'$ may fall below $\eps$ and then rise above it again. The assumptions $\beta\gamma<1$ and $\gamma<\beta$ ensure that the denominators in $R_u(t)$ and $R_{u'}(t)$ are nonzero. They also imply $\gamma^2<1$, which is equivalent to $\ell_u+\ell_{u'}>1$. This restriction is part of the sufficient conditions in the theorem; the same learning order may still hold when it is not satisfied.

\subsection{Numerical experiments}\label{sec:experiments}

\subsubsection{Verification of Theorem~\ref{thm:learning-time}}\label{sec:exp-learning-time}
We examine Theorem~\ref{thm:learning-time} on $\Omega=[-1,1]^2$ using tensor products of the one-dimensional basis families
\begin{align}
\mathcal B_{\mathrm L}&=\{P_r:r=0,\ldots,63\},\label{eq:exp-order-legendre}\\
\mathcal B_{\mathrm W}&=\left\{\frac{1}{\sqrt{2}}\right\}\cup
\left\{\frac{\cos(m\pi x)}{\sqrt{1+m}},\frac{\sin(m\pi x)}{\sqrt{1+m}}:m=1,\ldots,32\right\},\label{eq:exp-order-weighted}
\end{align}
where $P_r$ is the standard Legendre polynomial of degree $r$. The two classes contain $4096$ and $4225$ basis functions, respectively. We set $\vartheta=(1/2,1/2)$, $\eta=1/4$, $\eps=0.1$, and $c=1/2$, and train each target $f=\chi_{\Bk}=2^{-1}\ee^{\ii\Bk\cdot\bm x}$ separately from zero initialization. The largest Gram eigenvalues are $4$ and $1$, so the step size satisfies $\eta\le1/\lambda_1$ for both classes. We test the twelve frequencies
\begin{equation}
\begin{aligned}
\Bk/\pi\in\bigl\{&(0.5,0.5),(0.5,1.5),(1.5,1.5),(2.5,2.5),\\
                 &(0.5,5.5),(5.5,0.5),(5.5,5.5),(5.5,9.5),\\
                 &(10.5,10.5),(15.5,15.5),(24.5,24.5),(31.5,31.5)\bigr\}.
\end{aligned}
\end{equation}

Using the corresponding Gram spectra and local spectral weights, we maximize \eqref{eq:quantile-lower} and minimize \eqref{eq:harmonic-time} and \eqref{eq:quantile-time} separately over their admissible cutoffs among the distinct Gram eigenvalues. For example, at $\Bk=\pi(5.5,9.5)$, the right-hand sides of \eqref{eq:quantile-lower}--\eqref{eq:quantile-time} are $2166.975$, $4681.750$, and $6031.409$ for the Legendre class, and $516.224$, $1042.022$, and $1199.293$ for the weighted trigonometric class, respectively, to three decimal places. We then compute $T_\eps(\Bk)$ by full batch coefficient gradient descent, using the exact Gram matrix and target projections evaluated by 384-point Gauss--Legendre quadrature in each coordinate. The learning time is the first iteration satisfying $\rho_{\Bk}(t)\le\eps$, with the relative error measured against the attainable Fourier limit as in \eqref{eq:rho-T}.

Figure~\ref{fig:exp-learning-time} compares the measured learning times (black stars) with the two lower and two upper bounds in \eqref{eq:jensen-time}--\eqref{eq:quantile-time}. The bounds follow the variation in learning time across the sampled frequencies. Both lower bounds lie below $T_\eps(\Bk)$, and both upper bounds lie above it, for all $24$ runs, in agreement with Theorem~\ref{thm:learning-time}. At $\Bk=\pi(5.5,9.5)$, gradient descent gives $T_\eps(\Bk)=3090$ and $654$ for the Legendre and weighted trigonometric classes, respectively.

\begin{figure}[!htbp]
\centering
\includegraphics[width=.49\textwidth]{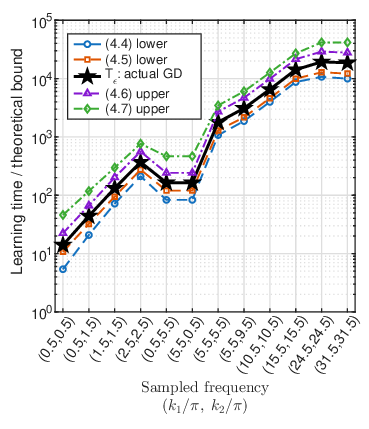}\hfill
\includegraphics[width=.49\textwidth]{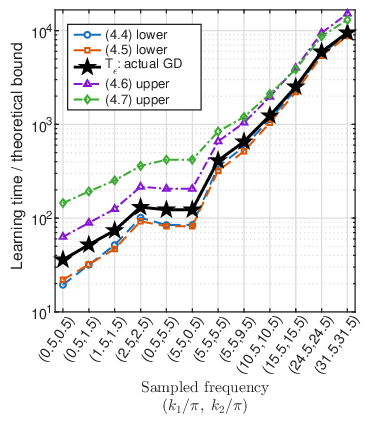}
\caption{Verification of Theorem~\ref{thm:learning-time} for the two-dimensional unnormalized Legendre basis (left) and weighted trigonometric basis (right).}
\label{fig:exp-learning-time}
\end{figure}
\FloatBarrier

\subsubsection{Verification of Theorem~\ref{thm:comparison}}\label{sec:exp-comparison}
\begin{itemize}
    \item Regions satisfying Theorem~\ref{thm:comparison}(a). 
    
    We use the Legendre and weighted trigonometric bases \eqref{eq:exp-order-legendre} and \eqref{eq:exp-order-weighted} on $\Omega=[-1,1]$. We set $\vartheta=1/2$, $\eta=1/2$, $\eps=0.1$, and $c=1/2$, and consider all frequency pairs in $\mathcal F^2$, where $\mathcal F=\{\pi(j+1/2):j=0,\ldots,31\}$. Each target $f=\chi_k=2^{-1/2}\ee^{\ii kx}$ is trained separately from zero initialization. We compute the optimized bounds $U_\eps$, $L_\eps$, and the learning times as in Section~\ref{sec:exp-learning-time}.

    Figure~\ref{fig:exp-order-regions} shows the sampled regions where $U_\eps(k)<L_\eps(k')$. This condition holds for $206$ of the $1024$ pairs with the Legendre class and $317$ pairs with the weighted trigonometric class. The colored cells all correspond to sampled pairs satisfying \(k<k'\). For each such pair, \(U_\eps(k)<L_\eps(k')\), and Theorem 4.2(a) gives \(T_\eps(k)<T_\eps(k')\), i.e., frequency $k$ can be learned faster. Thus both orderings hold for these pairs, although the experiment does not establish a causal relation between them.    
    \begin{figure}[!htbp]
    \centering
    \includegraphics[width=.49\textwidth]{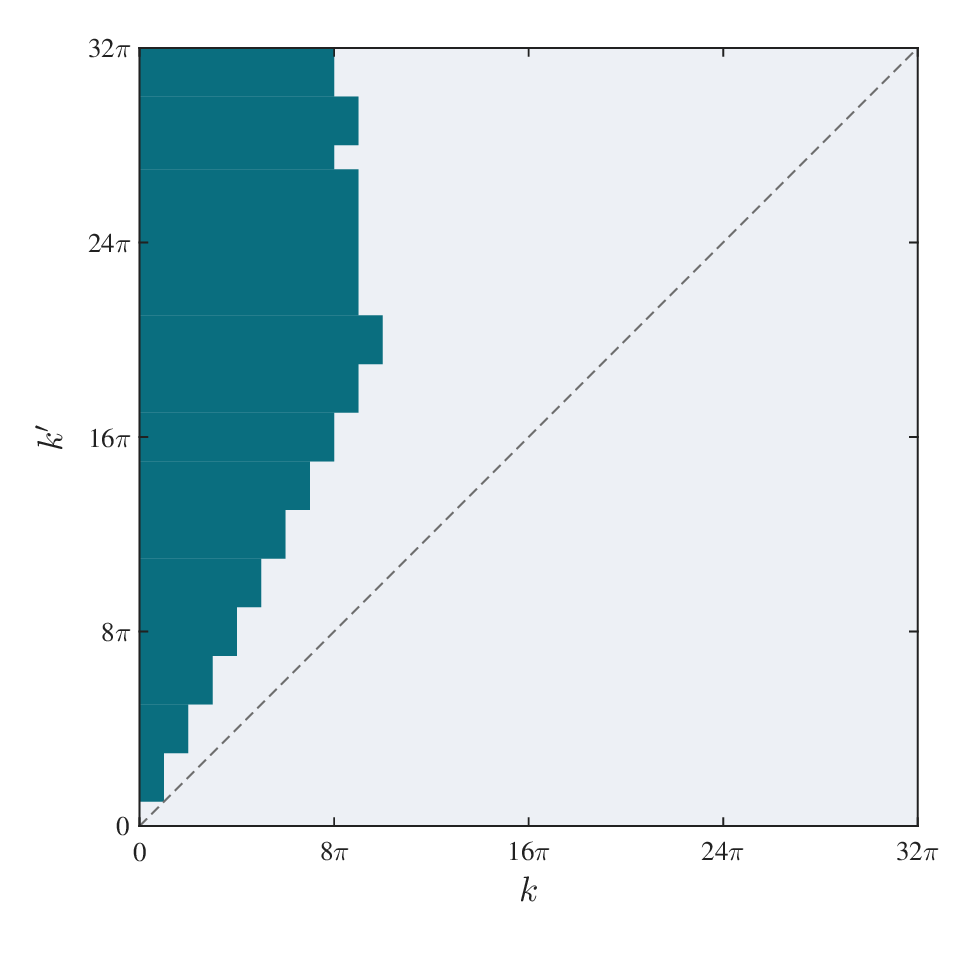}\hfill
    \includegraphics[width=.49\textwidth]{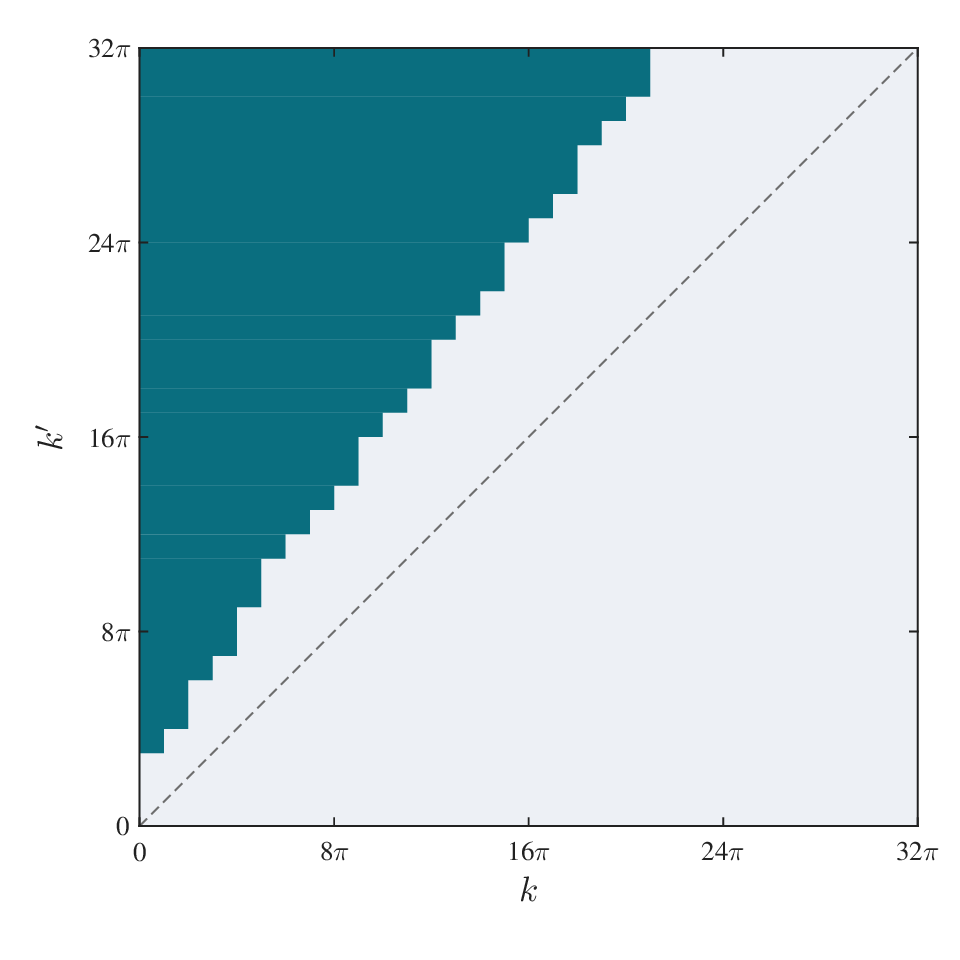}
    \caption{Frequency pairs satisfying $U_\eps(k)<L_\eps(k')$ (teal) for the Legendre class (left) and weighted trigonometric class (right).}
    \label{fig:exp-order-regions}
    \end{figure}    
    \item {Tests of Theorem~\ref{thm:comparison}(b) and (c).}
    
    We use the corresponding tensor product bases on $\Omega=[-1,1]^2$, containing $4096$ and $4225$ basis functions, respectively. We set $\vartheta=(1/2,1/2)$, $\eta=1/4$, $\eps=0.1$, $c=1/2$, and $p=0.8$ in Theorem~\ref{thm:comparison}(b). Each target $f=\chi_{\Bk}=2^{-1}\ee^{\ii\Bk\cdot\bm x}$ is trained separately from zero initialization. We test the two frequency pairs
    \begin{equation}
    \begin{array}{c|cc}
    \text{Case}&\Bk/\pi&\Bk'/\pi\\ \hline
    \mathrm A&(0.5,1.5)&(5.5,9.5)\\
    \mathrm B&(2.5,2.5)&(10.5,10.5)
    \end{array}
    \end{equation}
    
    Table~\ref{tab:exp-order-conclusions} reports $U_\eps(\Bk)$, $L_\eps(\Bk')$, and the measured learning times for both classes. Both frequency pairs satisfy the conditions in Theorem~\ref{thm:comparison}(b) and (c) for each class. In all four tests, the bounds satisfy $U_\eps(\Bk)<L_\eps(\Bk')$, and gradient descent gives $T_\eps(\Bk)<T_\eps(\Bk')$, in agreement with the predicted strict learning order.
    
    \begin{table}[!htbp]
    \centering
    \caption{Bounds and learning times.}
    \label{tab:exp-order-conclusions}
    \small
    \setlength{\tabcolsep}{4pt}
    \renewcommand{\arraystretch}{1.15}
    \begin{tabular}{@{}llccrrrr@{}}
    \hline
    Basis & Case & (b) & (c) & $U_\eps(\Bk)$ & $L_\eps(\Bk')$ & $T$ & $T'$\\
    \hline
    Legendre & A & Yes & Yes & 66.511 & 2166.975 & 44 & 3090\\
     & B & Yes & Yes & 560.567 & 4607.374 & 369 & 6522\\
    \hline
    Weighted trig. & A & Yes & Yes & 89.091 & 577.139 & 52 & 654\\
     & B & Yes & Yes & 216.035 & 1130.842 & 131 & 1240\\
    \hline
    \end{tabular}
    \end{table}
\end{itemize}

\subsubsection{Verification of Theorem~\ref{cor:two-modes}}\label{sec:exp-general-targets}
We examine Theorem~\ref{cor:two-modes} on $\Omega=[-1,1]$ using the unnormalized Legendre and weighted trigonometric bases \eqref{eq:exp-order-legendre} and \eqref{eq:exp-order-weighted}, containing $64$ and $65$ functions, respectively. We set $\vartheta=1/2$, $\eta=1/2$, and $\eps=0.1$, and train $f=u+u'$ from zero initialization, where $u=2^{-1/2}\ee^{\ii kx}$ and $u'=2^{-1/2}\ee^{\ii k'x}$. The frequency pairs $(k/\pi,k'/\pi)$ are $(0.5,9.5)$ in Case~A and $(2.5,15.5)$ in Case~B. The two modes are orthonormal, and $\eta\le1/\lambda_1$ holds for both classes. We use coefficient gradient descent with the exact Gram matrix and target projections evaluated by 384-point Gauss--Legendre quadrature. Each learning time is one plus the last iteration above the threshold, with a decreasing spectral envelope controlling the uncomputed tail.

Table~\ref{tab:exp-two-mode-order} reports the right-hand sides of \eqref{eq:two-low} and \eqref{eq:two-high} at $t_1$, together with the measured learning times. All four tests satisfy $\beta\gamma<1$ and $\gamma<\beta$, and both bounds hold at every computed iteration. At the listed $t_1$, the right-hand side of \eqref{eq:two-low} is below $\eps$, whereas that of \eqref{eq:two-high} is above $\eps$.  Thus the lower-frequency component reaches and remains below the prescribed relative error earlier, in agreement with Theorem~\ref{cor:two-modes}.

\begin{table}[!htbp]
\centering
\caption{Bounds and learning times for $f=u+u'$.}
\label{tab:exp-two-mode-order}
\small
\setlength{\tabcolsep}{5pt}
\renewcommand{\arraystretch}{1.15}
\begin{tabular}{@{}llrrrrr@{}}
\hline
Basis & Case & \eqref{eq:two-low} & \eqref{eq:two-high} & $T$ & $t_1$ & $T'$\\
\hline
Legendre & A & 0.077622 & 0.750207 & 6 & 9 & 111\\
 & B & 0.098914 & 0.518867 & 29 & 39 & 181\\
\hline
Weighted trig. & A & 0.096216 & 0.278414 & 8 & 21 & 47\\
 & B & 0.099132 & 0.361444 & 17 & 27 & 75\\
\hline
\end{tabular}
\end{table}

\section{Conclusion}\label{sec:conclusion}
We have studied a general frequency principle for gradient descent in finite linear approximation classes on a bounded box. Our results describe two complementary aspects of frequency-dependent learning. Under zero initialization and suitable Fourier decay assumptions, the output amplitude relative to the target amplitude vanishes at high frequencies along the prescribed frequency set. This holds at every fixed iteration and, under a stable step size, uniformly throughout training. Such small relative movement can reflect limited approximation capacity even when convergence to the attainable limit is fast.

To characterize convergence separately, we have derived learning-time bounds after removing the limiting approximation error. Comparing these bounds and the associated spectral quantities gives sufficient conditions for strict learning order between single-frequency targets trained separately. The analysis also gives conditions for two frequency components in the same run to reach and remain below a prescribed relative-error threshold in a specified order, accounting for their interaction.

The Legendre experiments illustrate the high-frequency movement bounds. Experiments with unnormalized Legendre and weighted trigonometric bases support the learning-time bounds and the predicted order for the tested pairs. The orthonormal periodic trigonometric example gives equal relative convergence rates for all represented frequencies. These examples show how the framework characterizes different learning behaviors according to the basis and its scaling.

Future work may focus on characterizing broader classes of basis functions for which the comparison condition also determines frequency order. Sharper learning-time bounds and estimates of interactions between components may extend the ordering results to targets with more than two components. Another direction is to study models whose basis functions evolve during training, where changes in the training operator must be incorporated into the comparison criteria

\bmhead{Acknowledgments}
This work was supported by the Major Research Plan of the National Natural Science Foundation of China (92370101).

\bmhead{Data availability}
The data and code supporting the numerical experiments in this study are available from the corresponding author upon reasonable request.

\bmhead{AI use declaration}
ChatGPT (OpenAI, GPT-6.1-sol) was used to generate the initial draft of the code for the numerical experiments and to identify the two-polynomial example presented in Eq.~\eqref{eq:two-polynomial-example}. The authors take full responsibility for the final code, the example, and the numerical results reported in this manuscript.

\bibliography{references_rev1}
\end{document}